\documentclass[10pt,twocolumn,letterpaper]{article}

\usepackage[pagenumbers]{wacv}

\usepackage{algorithm}
\usepackage{algpseudocode}
\usepackage{amsmath}
\usepackage{graphicx}
\usepackage{makecell}
\usepackage{amsthm}
\newtheorem{theorem}{Theorem}

\newtheorem{proposition}{Proposition}

\theoremstyle{remark}

\theoremstyle{definition}

\usepackage{adjustbox}
\usepackage{multirow}
\usepackage{xspace}

\newcommand{\Renyi}{R\'{e}nyi\xspace}

\makeatletter
\newcommand{\equationlabel}[1]{\gdef\df@label{#1}}
\makeatother
\definecolor{wacvblue}{rgb}{0.21,0.49,0.74}
\usepackage[pagebackref,breaklinks,colorlinks,allcolors=wacvblue]{hyperref}

\title{TADP: Dual-Branch Temporally Adaptive Differential Privacy for Expression-Aware Video Anonymization}

\author{Feng Xu$^{1,2}$, Xun Li$^{2}$, Lars Petersson$^{2}$, Yulei Sui$^{1}$, David Ahmedt-Aristizabal$^{2}$, Dadong Wang$^{2}$ \\
$^{1}$UNSW Sydney,
$^{2}$CSIRO, Australia\\
}

\begin{document}
\maketitle

\begin{abstract}
We present TADP, a temporally adaptive differential privacy (DP) framework for inference-time privacy-preserving video generation under fixed-length clip-replacement adjacency.
TADP uses a dual-branch design to extract identity and expression representations, clips and privatizes both branches with temporally adaptive Gaussian mechanisms whose noise scales depend only on preceding privatized releases within the same clip and are bounded below by history-uniform calibrated minima, and accounts for their joint release via adaptive \Renyi Differential Privacy (RDP) composition.
The clip-level RDP guarantee is converted to $(\varepsilon,\delta)$-differential privacy, while soft quantization and subsequent temporal processing are applied only after privatization as DP post-processing.
Experiments on facial expression datasets show reduced identity leakage under evaluated attacks, usable perceptual quality and temporal stability, and high facial expression recognition utility relative to the evaluated privacy-preserving baselines.
TADP formally satisfies clip-level DP and empirically provides a favorable privacy--utility trade-off under the evaluated threat models and datasets.

\end{abstract}

\section{Introduction} \label{intro}

Privacy preservation is increasingly important for visual data, driven by regulations such as the General Data Protection Regulation~\cite{GDPR2016} and the California Consumer Privacy Act~\cite{ccpa}, as well as growing user expectations. In domains such as healthcare, human-computer interaction, and social media~\cite{liu2017facial,faria2017affective,ahmedt2018motion,ahmedt2018deep,sonkusare2019detecting, yang2021undisturbed, serfaty2023generative, cantrill2024orientation, roshanzamir2024happens}, videos often contain sensitive facial identities but remain valuable for analysis and viewing. This creates a difficult requirement: privacy mechanisms must conceal identity without destroying visual content, temporal dynamics, or facial expression cues.

\begin{figure}[t]
   \centering
   \includegraphics[width=0.8\linewidth]{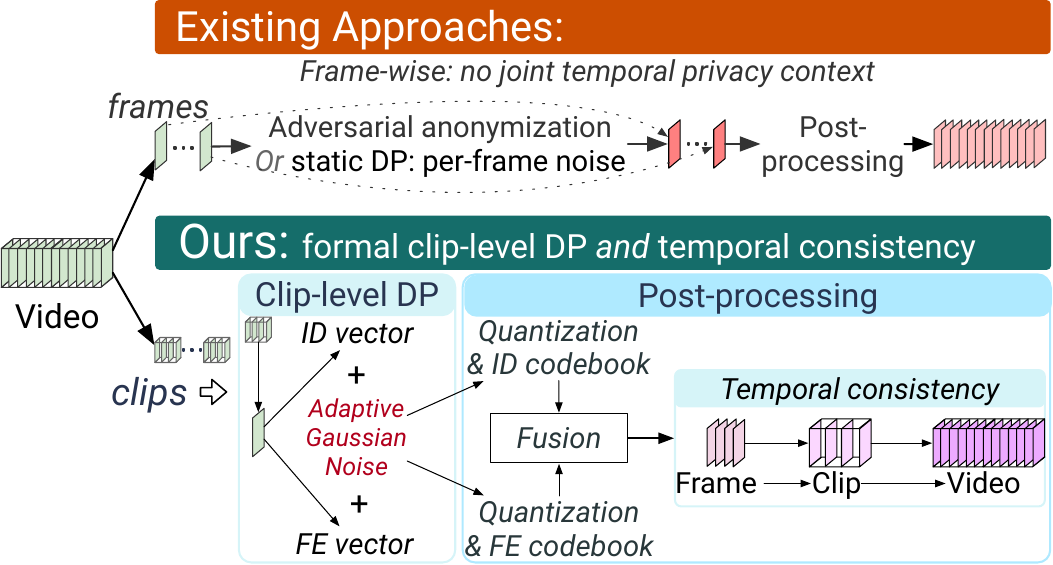}
   \caption{\textbf{TADP} jointly privatizes identity and expression representations over a fixed-length clip, providing a clip-level $(\varepsilon,\delta)$-DP guarantee. All subsequent temporal modeling and decoding operate only on privatized representations.} 
   \label{fig:overall}
\end{figure}

Formal differential privacy (DP) has been applied to visual anonymization. Image methods privatize disentangled identity representations or generative latents~\cite{wen2024differential,li2021differentially}, and provide protection through generative or pixel-space DP~\cite{croft2021obfuscation} and perceptual indistinguishability~\cite{chen2021perceptual}. 
Learning-based methods suppress sensitive attributes empirically but generally lack formal DP accounting~\cite{wu2020privacy,dave2022spact,fioresi2023ted,mi2023privacy,Xu_2025_WACV,mi2024minusface}.
For video, Video-DPRP reconstructs sample-wise DP videos for activity recognition~\cite{nken2025video}, while Fioresi \etal anonymize video embeddings without formal DP~\cite{fioresi2026privacy}. 
Our contribution lies in integrating these ingredients for expression-aware reconstruction with joint clip-level accounting of identity and expression releases, adaptation based only on privatized history, and post-privatization temporal modeling.
In this paper, video anonymization is the application setting, while the formal privacy unit is a fixed-length clip. This provides a clear clip-level guarantee for each processed segment without implying an unchanged whole-video DP guarantee for variable-length videos.

To this end, we study privacy-preserving video generation for facial expression analysis, aiming to:
(1) conceal the original facial identity while preserving temporal expression dynamics,
(2) maintain temporal consistency under a formal privacy guarantee, and
(3) avoid excessive perceptual degradation relative to the input video. 
We propose a temporally adaptive differential privacy (TADP) framework for identity-protected and expression-aware video generation, as illustrated in Figure~\ref{fig:overall}. TADP reconstructs videos from privatized latent representations using temporal contextualization and a frame-level generative decoder. 
Its privacy design applies branch-specific adaptive Gaussian mechanisms to identity and expression representations and assigns a smaller target privacy loss to the identity branch. Soft quantization is applied only after privatization.
Both branches are privatized because expression representations may still contain data-dependent and identity-correlated cues; privatizing only the identity branch would not justify an end-to-end DP release.
Quantization, fusion, temporal modeling, and decoding then operate only on privatized representations and constitute DP post-processing with no additional privacy loss.

We formalize privacy under fixed-length clip-level replacement adjacency, using a fixed $T$-frame clip as the accounting and release unit. Two clips are adjacent if one admissible clip is replaced by another satisfying the same public constraints, possibly with different identities and expression dynamics.
Within each clip, we compose the conditional RDP losses of the identity and expression branches within each frame and across all $T$ frames, then convert the resulting RDP curve to $(\varepsilon,\delta)$-DP. 
Releasing multiple clips from the same source video requires additional RDP composition before a single $(\varepsilon,\delta)$ conversion, with source-video duration treated as public.
The guarantee excludes training-membership privacy and metadata outside the mechanism. We separately evaluate video-level utility, temporal consistency, identity leakage, and component ablations.

In summary, our main contributions are:
\begin{itemize}
    \item A dual-branch DP mechanism that jointly accounts for identity and expression releases across a fixed-length clip using RDP composition.
    \item A temporally adaptive mechanism with noise scales conditioned only on privatized history, asymmetric branch-wise target-loss allocation, and downstream processing restricted to privatized representations as DP post-processing.
    \item Extensive evaluations of identity leakage under anonymization-aware attacks, facial-expression utility, perceptual quality, temporal consistency, and cross-domain stress testing, supported by ablations of the privacy and temporal components.
\end{itemize}

\section{Related Work}

\subsection{Diffusion Models}
Diffusion models~\cite{ho2020denoising,song2020denoising,ho2022video,rombach2022high,peebles2023scalable,chen2024videocrafter2,podell2024sdxl} have become a leading paradigm for image and video generation and editing. By reconstructing fine visual details from compact representations, they bridge high-level semantics and low-level fidelity, making them attractive for privacy-preserving video generation where perceptual quality matters. In our framework, however, SDXL is used as a frame-level synthesis decoder rather than as the source of temporal modeling: it maps each temporally contextualized privatized representation back to an image frame. 
Cross-frame and cross-clip consistency are instead modeled by an intra-clip Transformer and an inter-clip GRU operating only on privatized representations.

\subsection{Privacy Preservation Approaches}

\noindent\textbf{Formal privacy for visual reconstruction.}
DP bounds the distinguishability of adjacent inputs, while RDP supports sequential composition~\cite{kairouz2015composition,mironov2017renyi}. For images, prior methods privatize disentangled identity representations or generative latents~\cite{wen2024differential,li2021differentially}, and protect reconstructed outputs through generative or pixel-space DP~\cite{croft2021obfuscation} and perceptual indistinguishability~\cite{chen2021perceptual}. Thus, latent perturbation, identity-aware factorization, and protected image reconstruction are not themselves unique to our work.
For video, Video-DPRP reconstructs sample-wise DP videos through random projection for activity recognition~\cite{nken2025video}, while Fioresi \etal anonymize video embeddings without formal DP~\cite{fioresi2026privacy}. Our distinction lies in jointly accounting for identity and expression releases across a fixed-length clip, adapting noise only from preceding privatized outputs under history-uniform bounds, and restricting subsequent temporal modeling and decoding to privatized representations as DP post-processing. The novelty therefore concerns this combination rather than the general use of DP in visual latents or reconstructed videos.

\noindent\textbf{Learning-based privacy preservation.}
Non-DP methods learn privacy-preserving transformations or representations through supervised, adversarial, or self-supervised objectives~\cite{wu2020privacy,dave2022spact,fioresi2023ted,Xu_2025_WACV,xu2026knowledge}. Controlled~\cite{Xu_2025_WACV} suppresses identity while compensating expression features, Knowledge Priors~\cite{xu2026knowledge} targets identity-disentangled open-set video FER, and SPAct~\cite{dave2022spact} and TED-SPAD~\cite{fioresi2023ted} exploit video dynamics for privacy preservation. These methods provide empirical rather than formally accounted privacy. In contrast, our framework jointly accounts for identity and expression releases across frames and restricts temporal modeling and decoding to privatized representations as DP post-processing.

\section{Method}
\begin{figure*}[!ht]
    \centering
    \includegraphics[width=0.85\linewidth]{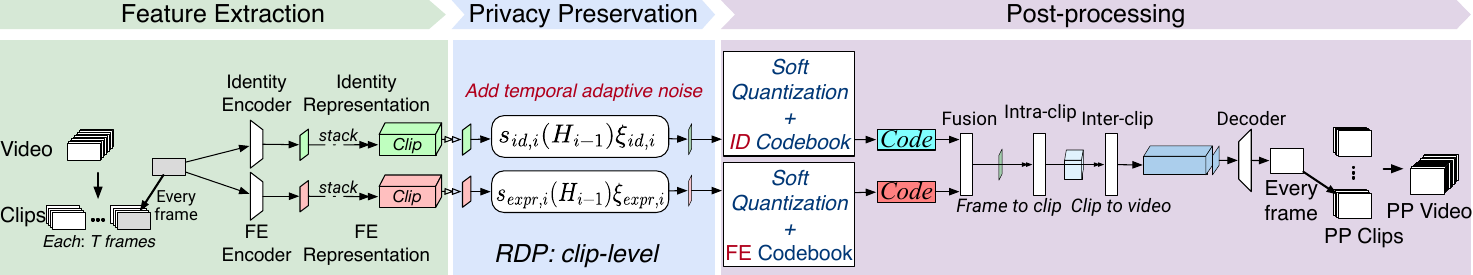}
    \caption{
    Overview of the TADP framework. The pipeline contains three stages: Feature Extraction, Privacy Preservation, and Post-processing. A video is divided into fixed-length clips, and every frame is processed by an Identity Encoder and an FE Encoder to obtain identity and facial-expression (FE) representations, which are stacked into clip-level features. In Privacy Preservation, each branch is clipped and privatized by temporally adaptive Gaussian noise, and the identity and FE RDP losses are composed to obtain the reported clip-level RDP guarantee. In Post-processing, the privatized identity and FE representations are mapped by Soft Quantization with the ID Codebook and FE Codebook into soft-quantized codebook representation, followed by Fusion, Intra-clip temporal modeling, Inter-clip memory, Frame to clip aggregation, Clip to video assembly, and a Decoder that produces PP Clips and the final PP Video. All post-privatization operations use only privatized representations and public metadata, incurring no additional privacy loss.
    }
    \label{fig:framework}
\end{figure*}

As shown in Figure~\ref{fig:framework}, TADP adopts a dual-branch architecture for expression-aware video anonymization. Given a clip $X=(x_1,\ldots,x_T)$, identity and expression encoders extract frame-level representations $z_i^{id}$ and $z_i^{expr}$. After branch-specific clipping, temporally adaptive Gaussian mechanisms privatize them with noise scales conditioned only on previously privatized outputs. The adaptive scale is learned on the training split and fixed before protected inference; it depends only on preceding privatized history and satisfies a history-uniform lower bound. Asymmetric target-loss allocation applies stronger noise to identity while preserving expression features. The privatized representations are softly quantized with separate codebooks, fused, and decoded by an SDXL decoder to generate the anonymized clip $\hat{X}$. For multi-clip videos, the post-processing decoder aggregates ordered privatized clip representations to maintain intra- and inter-clip temporal coherence. Since post-processing uses only DP representations and public metadata, it incurs no additional privacy loss. Thus, the privacy claim requires three conditions: both data-dependent branches are privatized, adaptive noise uses only previous DP outputs with a valid lower bound, and all downstream operations use only DP representations and public metadata. We next define the privacy setting, privatization mechanism, clip-level RDP guarantee, post-processing analysis, and training objectives.

\subsection{Problem Formulation and Privacy Setting}
We consider inference-time privacy for submitted clips, rather than training-set membership privacy. Since video length varies, whole-video DP under a fixed budget would require length-dependent per-frame noise, leading to inconsistent privacy--utility trade-offs. TADP uses a fixed $T$-frame clip as the accounting unit for one mechanism release, and jointly releasing multiple clips requires additional composition.

Let $X=(x_1,\ldots,x_T)$ and $X'=(x'_1,\ldots,x'_T)$ be two clips satisfying the same public preprocessing constraints. We use clip-level replacement adjacency: $X\sim X'$ if one complete preprocessed clip is replaced by another, possibly with different identities and expressions. A randomized mechanism $\mathcal{M}$ satisfies $(\alpha,\varepsilon)$-RDP if, for all $X\sim X'$,  $D_\alpha(\mathcal{M}(X)\|\mathcal{M}(X'))\leq\varepsilon$, where $\alpha>1$ and $D_\alpha$ is the order-$\alpha$ R\'enyi divergence~\cite{mironov2017renyi}. This bounds distinguishability between any two input clips, but does not cover training membership or metadata outside the mechanism.

Source-video duration is public. We split each video into non-overlapping $T$-frame clips in temporal order; if the final segment has fewer than $T$ frames, it is padded by repeating its last frame. Generated padded outputs are removed with the public padding mask, preserving the original duration. The reported $(\varepsilon,\delta)$ guarantee applies to one preprocessed clip; jointly releasing multiple clips from a source video requires composing their privacy losses, as discussed in~\cref{sec:privacy-guarantees}.

\subsection{Dual-Branch Temporally Adaptive Privatization}\label{sec:adaptive-dp}
Under the above privacy setting, the dual-branch mechanism must bound encoder sensitivity, privatize both data-dependent branches, and use temporal information without accessing unprivatized history. We therefore clip branch representations and apply branch-specific Gaussian mechanisms whose scales depend only on previously privatized outputs. The mechanism is applied independently to each preprocessed $T$-frame clip, and the history contains only privatized outputs from the same clip.

\noindent\textbf{Branch-wise Representation Clipping.}
For branch $b\in\{id,expr\}$, representation is clipped with threshold $C_b$:
\begin{align}
\bar z_i^b=z_i^b \min \! \left(1,\frac{C_b}{\|z_i^b\|_2}\right),
\label{eq:branch_clip}
\end{align}
where $C_b>0$ bounds the frame-level encoder output norm in branch $b$, not gradients. Features with norms at most $C_b$ are unchanged; the rest are rescaled to norm $C_b$. We use separate $C_{id}$ and $C_{expr}$ for their different norm distributions. Both are selected on the calibration split and fixed before inference. For adjacent clips, the per-step global sensitivity of branch $b$ is bounded by:
\begin{align}
\Delta_{b,i}:= \sup_{X\sim X'} \| \bar z_i^b-\bar z_i^{b \prime} \|_2 \leq 2C_b.
\label{eq:branch_sensitivity}
\end{align}
The full $T$-frame clip remains the private record. Per-step bounds are used only for sequential accounting within the clip.

\noindent\textbf{Temporally Adaptive Gaussian Privatization.}
Before step $i$, let the joint privatized history, $H_{i-1}$: $H_{i-1}=\left\{(\tilde z_j^{id},\tilde z_j^{expr})\right\}_{j=1}^{i-1}$. The mechanism adds Gaussian noise separately to each branch:
\begin{align}
\tilde z_i^b= \bar z_i^b+s_{b,i}(H_{i-1}) \xi_{b,i}, \qquad \xi_{b,i} \sim \mathcal{N}(0,I).
\label{eq:dual_branch_mechanism}
\end{align}
$s_{b,i}$ denotes the standard deviation applied independently to each coordinate of the $d_b$-dimensional representation, rather than the expected norm of the complete perturbation vector. Therefore, for $\xi_{b,i}\sim\mathcal{N}(0,I_{d_b})$, the perturbation norm concentrates approximately around $s_{b,i}\sqrt{d_b}$. We report both the per-coordinate scale and the realized perturbation norm in \cref{tab:noise_statistics,tab:signal-noise-statistics}.
The scale $s_{b,i}$ is adaptive, but privacy requires the history-uniform lower bound:
\begin{align}
s_{b,i}(H_{i-1}) \geq s_{b,\min,i}>0.
\label{eq:min_noise}
\end{align}
This is a hard privacy constraint: the adaptive module may increase noise for robustness or temporal behavior, but cannot reduce it below the calibrated minimum.

\noindent\textbf{Adaptive-scale Parameterization.}
The joint privatized history is summarized by a recurrent encoder. Specifically, the concatenated identity and expression releases from the preceding step are projected and used to update a recurrent state: $q_{i-1} = P_h[\tilde z_{i-1}^{id};\tilde z_{i-1}^{expr}]$, $h_{i-1} = \mathrm{GRU}(q_{i-1},h_{i-2})$, where $h_0$ is a fixed zero vector. For $i=1$, no privatized representation has yet been released, so the scale is predicted directly from $h_0$. For each branch $b$, a branch-specific MLP $g_b$ predicts the effective noise scale as: $s_{b,i}(H_{i-1})=s_{b,\min,i}\left[1+\text{softplus}\!\left(g_b(h_{i-1})\right)\right]$.
The Softplus output is nonnegative; consequently, this parameterization guarantees $s_{b,i}(H_{i-1})\geq s_{b,\min,i}$ for every possible privatized history. The scale predictor receives only preceding privatized releases and never accesses raw frames or unprivatized representations. Architectural and training details of this module are given in \cref{supp_adaptive_scale}.

\noindent\textbf{Asymmetric Branch-wise Target Privacy-loss Allocation.}
Identity features should be more strongly suppressed, while expression features should preserve task-relevant dynamics. 
We therefore allocate a target privacy loss asymmetrically to calibrate branch-specific noise lower bounds. Let $\rho_{id},\rho_{expr}>0$ and $\rho_{id}+\rho_{expr}=1$. For a target per-step RDP budget $\bar \epsilon_i(\alpha)$, set $\epsilon_{b,i}^{target}(\alpha)=\rho_b \bar \epsilon_i(\alpha)$ for $b \in \{id,expr\}$. The minimum noise is:
\begin{align*}
s_{b,\min,i} \geq \Delta_{b,i} \max_{ \alpha \in \mathcal{A}} \sqrt {\frac{\alpha}{2 \rho_b \bar \epsilon_i(\alpha)}}.
\end{align*}
Choosing $\rho_{id}<\rho_{expr}$ assigns less privacy loss and thus stronger noise to the identity branch, after accounting for sensitivity. The maximum over RDP orders makes the bound valid for the full accountant curve. 
%These coefficients calibrate branch-specific noise but do not weight privacy losses in composition.
These coefficients only calibrate the branch-specific noise lower bounds; the accountant still sums the actual identity and expression RDP losses without weighting.

\subsection{Privacy Guarantees}\label{sec:privacy-guarantees}
We now prove the guarantee of the mechanism in \cref{sec:adaptive-dp}. The analysis first gives a history-uniform per-step RDP bound for each branch-specific Gaussian release, then composes identity and expression losses within each frame and across the $T$-frame clip. The resulting clip-level RDP bound is finally converted to $(\varepsilon,\delta)$-DP.

\begin{theorem}[Conditional Per-step RDP for Each Branch]\label{thm:branch_rdp}
Fix a branch $b$, a time step $i$, and any possible privatized history $h_{i-1}$. Under \cref{eq:branch_sensitivity,eq:min_noise}, the conditional mechanism in \cref{eq:dual_branch_mechanism} satisfies $(\alpha,\epsilon_{b,i}(\alpha))$-RDP, where
\begin{align}
\epsilon_{b,i}(\alpha) \leq \frac{ \alpha \Delta_{b,i}^2}{2s_{b,i}(h_{i-1})^2} \leq \frac{ \alpha \Delta_{b,i}^2}{2s_{b,\min,i}^2}.
\label{eq:branch_rdp}
\end{align}
\end{theorem}

\begin{proof}
Conditioned on $H_{i-1}=h_{i-1}$, the adaptive scale is fixed and the mechanism is Gaussian with equal covariance under $X$ and $X'$. The RDP divergence between the two Gaussian distributions is $\alpha\|\bar z_i^b-\bar z_i^{b\prime}\|_2^2/(2s_{b,i}^2)$. Applying the sensitivity and minimum-noise bounds proves the result.
\end{proof}

\begin{theorem}[Dual-Branch Adaptive Clip Composition]\label{thm:rdp_composition}
Suppose the conditional bound in \cref{eq:branch_rdp} holds uniformly for every possible privatized history, every adjacent $X\sim X'$, both branches, and all $i\in\{1,\ldots,T\}$. The joint release of both branches for one complete $T$-frame clip satisfies:
\begin{align}
\epsilon_{\mathrm{clip}}(\alpha) \leq \sum_{i=1}^{T} \left[\epsilon_{id,i}(\alpha)+\epsilon_{expr,i}(\alpha)\right].
\label{eq:clip_rdp}
\end{align}
\end{theorem}

\begin{proof}
At each time step, releasing both data-dependent branches incurs the sum of their conditional RDP losses by composition. Applying adaptive RDP composition across the $T$ sequential releases gives \cref{eq:clip_rdp}. The result does not require the learned scales to be constant, but it requires a history-uniform upper bound derived from $s_{b,\min,i}$.
\end{proof}

Intuitively, conditioned on previous privatized history, each adaptive release is a Gaussian mechanism with fixed scale, and the history-uniform bound $s_{b,i}(h_{i-1})\ge s_{b,\min,i}$ gives a valid per-step RDP upper bound. The clip accountant sums branch losses per frame and composes them over $T$ frames.

The bound in \cref{eq:clip_rdp} applies to one $T$-frame clip, not automatically to the full source video. Under source-video replacement adjacency, if a video yields $K$ jointly released clips, a conservative guarantee composes across clips:
\begin{align*}
\epsilon_{\mathrm{long}}(\alpha)
\leq{}& \sum_{k=1}^{K}\epsilon_{\mathrm{clip},k}(\alpha) \notag\\
={}&\sum_{k=1}^{K}\sum_{i=1}^{T}
\left[\epsilon_{id,k,i}(\alpha)+\epsilon_{expr,k,i}(\alpha)\right].
\end{align*}
Converting $\epsilon_{\mathrm{long}}(\alpha)$ gives the corresponding source-video $(\varepsilon,\delta)$ guarantee. Thus, multiple clips cannot each use the full target budget while claiming the same budget for the whole video.

\noindent\textbf{Conversion to $(\varepsilon,\delta)$-DP.}
For evaluated RDP orders $\mathcal{A}$, we report:
\begin{align}
\varepsilon(\delta)=\min_{ \alpha \in \mathcal{A}} \left[\epsilon_{\mathrm{clip}}(\alpha) + \frac{\log (1/\delta)}{\alpha-1} \right].
\label{eq:rdp_conversion}
\end{align}
The accountant uses actual branch-specific lower bounds $s_{b,\min,i}$ and sensitivities $\Delta_{b,i}$, not only a nominal base scale such as $\sigma_0$. We report clipping thresholds, sequence length, allocation coefficients, effective-noise bounds, RDP orders, $\delta$, and the minimizing order in \cref{sec:exp}.

\subsection{Soft Quantization and Decoding as DP Post-processing}\label{sec:soft-quantization}

Before soft quantization, we apply a fixed, branch-specific rescaling: $r_b(\tilde z)=a_b\tilde z$, where $a_b$ is selected using only the training split and is fixed before protected inference. The same coordinate system is used for the corresponding codebook entries. Since the same scalar multiplies the clipped representation and its Gaussian perturbation, this operation does not change the noise-to-signal norm ratio and, as deterministic post-processing, does not change the privacy guarantee.

The previous subsection establishes privacy for the jointly privatized latent sequence, while the released output is a decoded video. We therefore ensure that all subsequent operations access only privatized representations and public metadata. At inference, no raw frames, unprivatized features, optical flow, identity embeddings, reference-image conditioning, encoder attention maps, or skip connections from the original video are provided to the decoder. Soft quantization, feature fusion, temporal modeling, and decoding are thus treated as post-processing.

We use $k$ to index clips, $i$ to index frames within a clip, and $b\in\{id,expr\}$ to index branches; codebook entries are indexed only by $j$. For each privatized branch output $\tilde z_{k,i}^b$, the codebook $\{e_j^b\}_{j=1}^{K_b}$ with temperature $\tau_b$ defines:
\begin{align*}
p_{b,j}^{(k,i)}= \frac{\exp(- \| \tilde z_{k,i}^b-e_j^b \|_2^2/ \tau_b)}{\sum_{j'}\exp(-\|\tilde z_{k,i}^b-e_{j'}^b\|_2^2/\tau_b)},  z_{q,k,i}^b=\sum_{j=1}^{K_b} p_{b,j}^{(k,i)}e_j^b.
% \label{eq:soft-assignment}
\end{align*}

This mapping is deterministic given $\tilde z_{k,i}^b$, so it is not an independent DP mechanism and receives no additional privacy budget. Empirically, soft quantization is useful because it maps noisy privatized vectors toward a learned representation manifold before temporal contextualization and decoding, stabilizing reconstruction after Gaussian perturbation. The temperature $\tau_b$ affects empirical smoothing and utility but not the accountant in \cref{eq:clip_rdp}. For each frame, the soft-quantized identity and expression vectors $z_{q,k,i}^{id}$ and $z_{q,k,i}^{expr}$ are fused as $u_{k,i}=\mathrm{Fuse}\!\left(z_{q,k,i}^{id},z_{q,k,i}^{expr}\right)$. 
Intra-clip consistency is modeled by a Temporal Transformer Encoder: $h_{k,1:T}=\mathrm{TransEnc}_{clip}\!\left(u_{k,1:T}+p_{1:T}\right)$, where $p_i$ is a fixed positional embedding and $h_{k,i}$ remains a per-frame representation. Inter-clip consistency is modeled by a boundary-aware GRU. Each clip is summarized as $c_k=\mathrm{AttnPool}(h_{k,1:T})$, the previous boundary as $b_{k-1}=\mathrm{Pool}(h_{k-1,T-L+1:T})$, and the memory is updated by $m_k=\mathrm{GRU}\!\left([c_k;b_{k-1}],m_{k-1}\right)$, with $b_0$ and $m_0$ fixed public constants. The memory is fused back before decoding: $\bar h_{k,i}=\mathrm{MLP}\!\left([h_{k,i};m_k]\right)$.
Accordingly, the released video can be written as $\hat V=F_\theta\!\left(\{Z_{q,k}^{id},Z_{q,k}^{expr}\}_{k=1}^{K}\right)$, where $Z_{q,k}^{b}=(z_{q,k,1}^{b},\ldots,z_{q,k,T}^{b})$. The temporal modules contextualize per-frame DP vectors, while the final generator remains a frame-level SDXL decoder. A projection maps each contextualized vector to the SDXL latent space before frame-wise decoding: $\ell_{k,i}=P_{\theta}(\bar h_{k,i})$, $\hat x_{k,i}=D_{\mathrm{SDXL}}(\ell_{k,i})$.
For reproducibility, the identity encoder, expression encoder, branch-specific codebooks, and SDXL decoder are initialized from the pretrained models described in \cref{sec:exp} and kept frozen. The adaptive-noise scale module, branch-fusion module, Temporal Transformer Encoder, boundary-aware GRU, attention/MLP temporal fusion, and projection layer $P_{\theta}$ are trainable. Thus, raw training videos supervise only the trainable adaptation and post-privatization modules through the losses in \cref{sec:training_loss}; they do not create inference-time conditioning paths from raw frames or unprivatized features.

\begin{proposition}[Post-processing of the Joint DP Representation]\label{prop:postprocessing}

Let $\mathcal{M}_{1:T}(X)$ denote the jointly privatized identity and expression sequences for one clip satisfying $(\alpha,\epsilon_{\mathrm{clip}}(\alpha))$-RDP. Let $F$ contain soft quantization, temporal smoothing, attention fusion, and decoding, and assume that $F$ receives no additional data-dependent input from the original video. Then $\hat X=F\!\left(\mathcal{M}_{1:T}(X)\right)$ also satisfies $(\alpha,\epsilon_{\mathrm{clip}}(\alpha))$-RDP and the corresponding $(\varepsilon,\delta)$-DP guarantee. For a source video with $K$ jointly released privatized clips, the same post-processing argument applies to the video-level decoder $F_\theta$ after the $K$ clip-level releases have first been composed according to \cref{sec:privacy-guarantees}.
\end{proposition}

\begin{proof}
The result follows from the data-processing inequality for \Renyi divergence. Deterministic or randomized functions of DP outputs do not increase privacy loss.
\end{proof}

This proposition requires all codebooks, fusion modules, decoder parameters, and inference-time conditioning signals to be public or fixed independently of the protected video. In particular, no raw-frame or unprivatized-latent skips, optical-flow inputs, encoder attention maps, identity embeddings, or reference-image conditioning paths enter the decoder. Inter-clip memory, boundary states, blending, and padding removal use only DP representations, generated frames, and public metadata. For longer videos, these post-processing operations do not add to the composed loss, but a single-clip guarantee does not become a source-video guarantee.

\subsection{Training Objectives}\label{sec:training_loss}

The proposed framework jointly optimizes the trainable components of the framework. The adaptive-noise module is trained end-to-end: from previous privatized history $H_{i-1}$, it predicts only a nonnegative multiplier above the calibrated lower bound $s_{b,\min,i}$. The losses below shape this multiplier through reparameterized Gaussian sampling, while the privacy accountant still uses the fixed lower bounds rather than learned average noise. The overall objective is:
\begin{equation*}
\mathcal{L}=\lambda_{exp}\mathcal{L}_{exp} + \lambda_{rec}\mathcal{L}_{rec} + \lambda_{perc}\mathcal{L}_{perc} + \lambda_{temp}\mathcal{L}_{temp}.
\end{equation*}
Training may use multiple clips from a source video, but the inference-time privacy unit remains one $T$-frame clip, with additional composition for jointly released clips.

\noindent\textbf{Expression Loss.}
We preserve expression semantics using $\mathcal{L}_{exp}=\mathrm{CE}(f_{exp}(\hat{x}),y)$, where $f_{exp}$ is a fixed expression-recognition network.

\noindent\textbf{Reconstruction Loss.}
We use $\mathcal{L}_{rec}=\|\hat{x}-x\|_1$ to preserve low-level content. This does not contradict identity suppression because, at inference, any identity information reaching the decoder must pass through clipping and DP privatization. This training objective does not alter the inference-time privacy accountant; the released inference path must still satisfy the restrictions in \cref{prop:postprocessing}.

\noindent\textbf{Temporal Consistency Losses.}
Temporal regularization is applied at both the clip and video levels, while the released inference path still uses only privatized or post-processed representations. We decompose: $\mathcal{L}_{temp}=\lambda_{intra}\mathcal{L}_{intra}+\lambda_{inter}\mathcal{L}_{inter} +\lambda_{bd}\mathcal{L}_{bd}.$
The intra-clip term preserves local smoothness within each $T$-frame clip: $\mathcal{L}_{intra}=\sum_{k=1}^{K}\sum_{i=2}^{T}\sum_{b\in\{id,expr\}} \|z_{q,k,i}^{b}-z_{q,k,i-1}^{b}\|_2^2$.
The inter-clip term encourages continuity between the last privatized quantized representation of one clip and the first representation of the next: $\mathcal{L}_{inter} =\sum_{k=2}^{K}\sum_{b\in\{id,expr\}} \|z_{q,k,1}^{b}-z_{q,k-1,T}^{b}\|_2^2$.
Finally, the boundary term regularizes the generated video itself at clip transitions using a fixed perceptual feature extractor $\psi$: $\mathcal{L}_{bd}=\sum_{k=2}^{K} \|\psi(\hat{x}_{k,1})-\psi(\hat{x}_{k-1,T})\|_1$.
For a single isolated clip, $\mathcal{L}_{inter}$ and $\mathcal{L}_{bd}$ are omitted. These losses may use original frames as supervision during training, but at inference neither raw frames, raw optical flow, nor unclipped private features are provided to the decoder.

\noindent\textbf{Perceptual Loss.}
We use $\mathcal{L}_{perc}=\|\phi(\hat{x})-\phi(x)\|_2^2$, where $\phi$ is a fixed perceptual encoder. The loss is used during model training; at inference, $\phi(x)$ is not provided to the decoder.

Multi-clip batches are sampled as ordered consecutive clips from the same source video, with the GRU memory reset at the window start. Isolated or shuffled clips are treated as single-clip samples: the memory is reset to $m_0$, and $\mathcal{L}_{inter}$ and $\mathcal{L}_{bd}$ are omitted. The numerical loss weights and the exact training protocol are reported in \cref{sec:exp}.

\section{Experiment}\label{sec:exp}

\subsection{Experimental Settings}

The experimental protocol is organized around the datasets and splits, formal privacy-accounting configuration, implementation, comparison methods, and attacker threat models described below.

\subsubsection{Datasets}\label{sec:supp_dataset}
We conduct experiments on two facial-expression video datasets with identity ground truth, CREMA-D~\cite{cao2014crema} and RAVDESS~\cite{livingstone2012ryerson}. 
CREMA-D contains 7,442 clips from 91 actors with six facial-expression categories. We use 7,440 clips because one clip contains only two frames and another contains several concatenated clips. RAVDESS contains 1,440 clips from 24 actors with eight facial-expression categories. 
We use an actor-disjoint 75\%-25\% split into development and test sets. The 75\% development set is further divided into actor-disjoint training and calibration subsets using a 90\%-10\% split. The calibration subset is used only to select the clipping thresholds and other fixed privacy-mechanism parameters, while the test set remains untouched until final evaluation. The random seed is set as 42.

\subsubsection{Complete Privacy-Accounting Configuration}\label{sec:supp_accountant}
We evaluate inference-time clip-level replacement privacy for the joint release of identity and expression representations. This fixed-length unit provides a stable basis for privacy--utility comparison; it does not preclude a separately composed source-video guarantee. For each branch and frame, the accountant uses the clipping bound $\Delta_{b,i}\leq2C_b$ and the guaranteed effective-noise lower bound $s_{b,\min,i}$. The per-frame losses of both branches are added and then composed across the clip length $T$ according to \cref{thm:rdp_composition}. We convert the resulting RDP curve using \cref{eq:rdp_conversion}. A nominal base scale is not used as a substitute for the effective branch-specific bounds.

For full reproducibility, \cref{tab:supp_accountant_parameters} provides the complete calibration and accounting configuration for each dataset. Clipping thresholds and all quantities derived from them are calibrated independently on the actor-disjoint calibration split of each dataset. No test sample is used to select these parameters.

\begin{table*}[ht]
\centering
\caption{Complete dataset- and branch-specific privacy-accounting configuration.}
\resizebox{\textwidth}{!}{
\begin{tabular}{ll|c|c|c}
\toprule
Dataset & Parameter & Identity branch & Expression branch & Shared/clip-level \\
\midrule
\multirow{12}{*}{CREMA-D}
& Clipping percentile & 95th & 95th & -- \\
& Clipping norm & $C_{id}=0.539$ & $C_{expr}=0.844$ & -- \\
& Calibration clipping rate & $r_{id}^{cal}=0.05$ & $r_{expr}^{cal}=0.05$ & -- \\
& Test clipping rate & $r_{id}^{test}=0.0659$ & $r_{expr}^{test}=0.1617$ & -- \\
& Sensitivity bound & $2C_{id}=1.078$ & $2C_{expr}=1.688$ & -- \\
& Target-loss allocation & $\rho_{id}=0.4$ & $\rho_{expr}=0.6$ & $\rho_{id}+\rho_{expr}=1$ \\
& Guaranteed effective-noise lower bound & $s_{id,\min}=8.921$ & $s_{expr,\min}=11.406$ & -- \\
& Observed effective-noise range & $s_{id,i} \in [8.921,14.348]$ & $s_{expr,i} \in [11.406,13.024]$ & -- \\
& Soft-quantization temperature & $\tau_{id}=0.5$ & $\tau_{expr}=0.8$ & Excluded from accountant \\
& Evaluated RDP orders $\mathcal{A}$ & -- & -- & $\mathcal{A}=\{1.5,2,3,4,8,16,32\}$ \\
& Minimizing RDP order $\alpha^\star$ & -- & -- & 8 \\
& Achieved clip-level $\varepsilon$ & -- & -- & 3.9809 \\
\midrule
\multirow{12}{*}{RAVDESS}
& Clipping percentile & 95th & 95th & -- \\
& Clipping norm & $C_{id}=0.678$ & $C_{expr}=0.921$ & -- \\
& Calibration clipping rate & $r_{id}^{cal}=0.05$ & $r_{expr}^{cal}=0.05$ & -- \\
& Test clipping rate & $r_{id}^{test}=0.0417$ & $r_{expr}^{test}=0.1250$ & -- \\
& Sensitivity bound & $2C_{id}=1.356$ & $2C_{expr}=1.842$ & -- \\
& Target-loss allocation & $\rho_{id}=0.3$ & $\rho_{expr}=0.7$ & $\rho_{id}+\rho_{expr}=1$ \\
& Guaranteed effective-noise lower bound & $s_{id,\min}=17.088$ & $s_{expr,\min}=15.197$ & -- \\
& Observed effective-noise range & $s_{id,i} \in [17.088,25.107]$ & $s_{expr,i} \in [15.197,18.834]$ & -- \\
& Soft-quantization temperature & $\tau_{id}=0.5$ & $\tau_{expr}=0.8$ & Excluded from accountant \\
& Evaluated RDP orders $\mathcal{A}$ & -- & -- & $\mathcal{A}=\{1.5,2,3,4,8,16,32\}$ \\
& Minimizing RDP order $\alpha^\star$ & -- & -- & 8 \\
& Achieved clip-level $\varepsilon$ & -- & -- & 2.9880 \\
\midrule
\multicolumn{2}{l|}{Accounting clip length and stride} & -- & -- & $T=16$, stride $=16$ \\
\multicolumn{2}{l|}{Failure probability} & -- & -- & $\delta = 10^{-5}$ \\
\multicolumn{2}{l|}{Random seed} & -- & -- & 42 \\
\bottomrule
\end{tabular}}
\label{tab:supp_accountant_parameters}
\end{table*}
For both datasets, the clipping thresholds are selected at the 95th percentile on the actor-disjoint calibration split and then fixed before test evaluation. \Cref{tab:supp_accountant_parameters} reports the resulting clipping rates, target-loss allocations, guaranteed noise floors, observed scale ranges, and achieved clip-level privacy losses. The observed ranges are diagnostic only: the formal accountant uses the guaranteed lower bounds. Soft-quantization temperatures are likewise reported for reproducibility but are excluded from the privacy accountant.

\subsubsection{Source-video Privacy Composition}\label{sec:supp_video_composition}
For a source video producing $K$ jointly released clips, we compose the complete clip-level RDP curve before conversion:
\begin{align}
\epsilon_{\mathrm{video}}(\alpha)
&=\sum_{k=1}^{K}\epsilon_{\mathrm{clip},k}(\alpha),
\label{eq:supp_video_rdp_composition}\\
\varepsilon_{\mathrm{video}}(\delta)
&=\min_{\alpha\in\mathcal A}
\left[
\epsilon_{\mathrm{video}}(\alpha)
+\frac{\log(1/\delta)}{\alpha-1}
\right].
\label{eq:supp_video_dp_conversion}
\end{align}
We evaluate only $\delta=10^{-5}$; all reported conversions use the RDP orders $\mathcal A=\{1.5,2,3,4,8,16,32\}$. The conversion is performed once after RDP composition; consequently, $\varepsilon_{\mathrm{video}}$ is not obtained by simply multiplying the final per-clip $\varepsilon$ by $K$.

After deterministic preprocessing, CREMA-D source videos produce between 1 and 10 non-overlapping 16-frame clips, with a mean of 4.37 clips per source video. RAVDESS source videos produce between 4 and 7 clips, with a mean of 5.19. Because the number of clips must be an integer for an individual video, the mean describes only the dataset length distribution; it is not itself used as a privacy-composition count. We therefore report the observed minimum and maximum together with the nearest integer to the dataset mean as a representative example.

\begin{table}[t]
\centering
\caption{Dataset source-video clip-count statistics and conservative privacy guarantees obtained by composing non-overlapping 16-frame clip releases. ``Near mean'' is the integer clip count nearest to the dataset mean.}
\label{tab:video_privacy_composition}
\resizebox{0.9\linewidth}{!}{
\begin{tabular}{l|l|c|c|c}
\toprule
Dataset & Length statistic & $K$ & $\alpha^\star$ & $\varepsilon_{\mathrm{video}}$ \\
\midrule
\multirow{3}{*}{CREMA-D}
& Minimum & 1 & 8 & 3.981 \\
& Near mean (mean $=4.37$) & 4 & 4 & 8.510 \\
& Maximum & 10 & 3 & 14.517 \\
\midrule
\multirow{3}{*}{RAVDESS}
& Minimum & 4 & 4 & 6.524 \\
& Near mean (mean $=5.19$) & 5 & 4 & 7.196 \\
& Maximum & 7 & 4 & 8.539 \\
\bottomrule
\end{tabular}}
\end{table}

These values demonstrate that privacy loss increases with the number of jointly released clips, although RDP composition followed by a single conversion is tighter than naively multiplying the final per-clip $\varepsilon$. The formal guarantee reported in the primary comparison therefore applies only to one 16-frame release. Because source-video duration is public, the guarantee for an individual source video is selected using its actual integer clip count after deterministic preprocessing.

\subsubsection{Implementation Details}\label{sec:supp_implementation}
Input frames are normalized to $[0,1]$. We extract 1024-dimensional expression and identity features using an AffectNet-based encoder~\cite{mollahosseini2017affectnet} without its classification head and ArcFace~\cite{deng2019arcface}, respectively. The ID and expression codebooks are pre-trained on CelebA~\cite{liu2015faceattributes} and AffectNet, respectively. 
We set the accounting clip length to $T=16$ and treat source-video duration as public. Longer videos are divided into non-overlapping 16-frame clips, whereas shorter clips are padded by repeating the final frame. The same deterministic preprocessing is used for training and evaluation. After frame-wise soft quantization and branch fusion, a Temporal Transformer Encoder supplies intra-clip context, a boundary-aware GRU memory supplies inter-clip context, and a frame-level SDXL decoder, implemented as a 2D image decoder, generates $T$ frames for each privatized and quantized clip sequence. When the last clip contains repeated padding frames, the generated outputs at those padded positions are removed according to the public padding mask before the full video is released. The boundary-aware GRU is used only when multiple released clips from the same source video are assembled into a complete video.

\paragraph{Adaptive-scale module.}\label{supp_adaptive_scale}
The history projection maps the concatenated 2,048-dimensional identity--expression (1024 each) release to 256 dimensions. The history encoder is a single-layer GRU with hidden size 256. Two branch-specific scale-prediction MLPs are used for the identity and expression branches; their parameters are not shared. Each MLP has dimensions $256\rightarrow128\rightarrow1$ and uses a GELU activation between the two linear layers. The scalar output is converted to the effective noise scale by the Softplus parameterization: $s_{b,i}(H_{i-1})=s_{b,\min,i}\left[1+\text{softplus}\!\left(g_b(h_{i-1})\right)\right]$.

The initial recurrent state is a fixed zero vector. For both branch-specific MLPs, the weights of the final linear layer are initialized to zero and its bias is initialized to $-5$. Hence, at the start of training, $\text{softplus}(-5)\approx0.0067$ and the effective scale is close to, but never below, the calibrated lower bound. The remaining trainable weights use Xavier initialization.
The adaptive-scale module is trained jointly with the fusion and temporal post-processing modules through reparameterized Gaussian sampling. Gradients from the expression, reconstruction, perceptual, and temporal objectives therefore shape the learned scale schedule, while the hard parameterization preserves the calibrated lower bound. No additional regularization term is applied directly to the predicted noise scale.
After training, all adaptive-scale parameters are fixed before protected inference. At inference, the module receives exclusively preceding privatized identity and expression representations. 
% The privacy accountant always uses $s_{b,\min,i}$ and does not rely on the learned mean or observed scale distribution. 
Realized scale statistics are reported in~\cref{tab:noise_statistics}.

We optimize the framework using AdamW~\cite{loshchilov2017decoupled} with a cosine-decay learning-rate schedule. The initial learning rate is \(0.001\), the batch size is 16, and training lasts for 100 epochs. Each codebook contains 1,000 entries of dimension 1,024. We set $\lambda_{exp}=\lambda_{rec}=\lambda_{perc}=\lambda_{temp}=\lambda_{intra}=\lambda_{inter}=\lambda_{bd}=0.5$. The training objective does not include an additional codebook loss. Training takes approximately 24 hours on four NVIDIA H100 GPUs.

After training, we perform inference on both the training and test splits to generate privacy-preserved videos. 
%The inferred results are used to validate the facial expression features of the generated videos.
The generated videos are used for downstream FER evaluation.
For each anonymization method, R3D, TimeSformer, and VideoMAE V2 are separately trained on the corresponding privacy-preserved training videos and evaluated on the corresponding privacy-preserved test videos.
FER is evaluated on each complete source video rather than by copying its label to every 16-frame clip. This preserves the original task definition because the onset, apex, and offset of an expression may occur in different clips. By contrast, cPLR operates on individual 16-frame releases and measures clip-level empirical identity leakage. The formal guarantee for a complete video release is obtained by composing the RDP curves of its jointly released clips as described in~\cref{sec:supp_video_composition}.

\subsubsection{Baseline Comparison}\label{sec:supp_baselines}
We compare our method with six privacy-preserving baselines: the adversarial approach SPAct~\cite{dave2022spact}, the controlled privacy-preservation method of~\cite{Xu_2025_WACV} (``Controlled''), the face-swapping method MobileFaceSwap~\cite{xu2022mobilefaceswap} (``Swapping''), the image-based method MinusFace~\cite{mi2024minusface}, Knowledge Priors for identity-disentangled open-set video FER~\cite{xu2026knowledge}, and Pi-Net, which provides perceptual indistinguishability~\cite{chen2021perceptual}. For a fair utility comparison, we separately train R3D, TimeSformer, and VideoMAE V2 on the privacy-preserved training videos produced by each method and evaluate them on the corresponding privacy-preserved test videos.

\subsubsection{Evaluation Metrics}\label{sec:supp_evaluation}
Our framework is evaluated on three aspects: utility task performance, privacy preservation effectiveness, and temporal consistency. We compare our method against baselines on the CREMA-D and RAVDESS datasets.

For evaluating temporal consistency, we report inter-frame flow-guided warping error (FWE), Fréchet video distance (FVD) using I3D~\cite{carreira2017quo} to extract features, temporal LPIPS (tLPIPS) as the average LPIPS distance between consecutive generated frames, and temporal identity consistency (T-ID) using ArcFace to extract features. 
For utility, we employ facial expression recognition (FER) accuracy as the main metric.

\subsubsection{Closed-set Identity Leakage Attacks} \label{sec:supp_identity_attacks}

We evaluate direct identity leakage at both frame and clip levels. For this empirical attack evaluation, the held-out test actors remain disjoint from the actors used to train TADP and the anonymization baselines. The videos of each held-out actor are further divided into an attacker-enrollment subset and a disjoint probe subset. Thus, the anonymization models never use the held-out actors during training, whereas the closed-set attacker is granted labeled enrollment samples for every probe identity.

Frame-level PLR follows~\cite{Xu_2025_WACV}. A ResNet50 identity attacker is trained using the privacy-preserved enrollment frames and evaluated on randomly sampled frames from the corresponding privacy-preserved probe videos. PLR is its top-1 identity recognition accuracy.

To evaluate identity leakage at the same granularity as the formal privacy unit, we additionally introduce clip privacy leakage ratio (cPLR). The cPLR attacker uses InternVideo2~\cite{wang2024internvideo2} to encode complete non-overlapping 16-frame clips. For each anonymization method and dataset, a separate InternVideo2 attacker is trained on the privacy-preserved attacker-enrollment clips produced by that method and evaluated on its disjoint privacy-preserved probe clips. Consequently, the attacker is aware of the output distribution of the evaluated anonymization mechanism. cPLR is defined as the resulting closed-set top-1 identity-recognition accuracy. Lower PLR and cPLR indicate less residual identity information at the frame and clip levels, respectively.
Because cPLR evaluates an individual 16-frame release, it should not be interpreted as source-video-level leakage after aggregating multiple released clips. An attacker that aggregates all clips of a source video may be stronger and should be evaluated together with the composed source-video privacy budget.

\subsubsection{Open-set Identity Identification and Verification}\label{sec:supp_open_set_identity_attack}
The closed-set PLR and cPLR evaluations assume that every probe identity is registered by the attacker. To relax this assumption, we adapt a separate AdaFace-based identity matcher~\cite{kim2022adaface} for each anonymization method using only that method's privacy-preserved outputs from the development actors. The architecture, optimization schedule, and training budget are identical across methods, and all actors in the final test split remain unseen during attacker adaptation.

For open-set verification, embeddings from the original enrollment clips of each enrolled test actor are averaged and $\ell_2$-normalized to form a template. Privacy-preserved or RVRT-recovered 16-frame clips serve as probes, and their representations are obtained by averaging normalized frame embeddings. Identity similarity is measured by cosine similarity. Thresholds are selected exclusively from genuine and impostor pairs among development actors and fixed before test evaluation. We report true accept rate (TAR) at false accept rates (FAR) of $10^{-2}$ and $10^{-3}$ and equal error rate (EER); lower TAR and higher EER indicate stronger identity privacy.

For open-set identification, the test actors are divided into enrolled and unknown identities. The gallery contains one template for each enrolled identity, while probes include both enrolled and unknown actors. A probe is identified only when its maximum gallery similarity exceeds a threshold selected using the development actors. We report detection and identification rate (DIR) at FAR values of $10^{-2}$ and $10^{-3}$. To reduce sensitivity to the limited number of test actors, we repeat the enrolled/unknown partition using five random seeds and report mean$\pm$standard deviation.

\subsection{Results}

\noindent\textbf{Utility vs. Privacy.}
\Cref{tab:FER_PLR} shows that TADP offers a strong empirical privacy--utility trade-off among the evaluated baselines across datasets and FER backbones, combining low PLR and cPLR with high FER accuracy; FER-only is the non-private upper bound.
The reported noise scale is a per-coordinate standard deviation, so the perturbation norm can greatly exceed the clipped-feature norm. Nevertheless, input interventions and repeated releases show that full-video FER depends on the correctly aligned sequence of clip-specific privatized evidence and remains stable across independent mechanism randomness, rather than being explained solely by decoder/classifier priors or a favorable noise realization. Stage-wise probes and soft-quantization diagnostics are provided in \cref{stage_wise}.

\begin{table*}[!t]
    \centering
    \caption{Full-video FER under $K$-specific composed budgets, frame-level PLR, and 16-frame cPLR (\%); see~\cref{sec:supp_video_composition}.}
    \resizebox{0.95\linewidth}{!}{
    \begin{tabular}{c|c|c|c|c|c|c|c|c|c|c}
    \toprule
         \multirow{2}{*}{Method}& \multicolumn{2} {c|}{Full-video} R3D FER Acc.$\uparrow$ & \multicolumn{2} {c|}{Full-video TimeSformer FER Acc.$\uparrow$}& \multicolumn{2} {c|}{Full-video VideoMAE V2 FER Acc.$\uparrow$}& \multicolumn{2} {c|}{Frame-level PLR$\downarrow$} & \multicolumn{2} {c}{16-frame cPLR$\downarrow$}\\ \cline{2-11}
         & CREMA-D& RAVDESS&CREMA-D& RAVDESS&CREMA-D & RAVDESS&CREMA-D &RAVDESS &CREMA-D &RAVDESS\\
    \midrule
        FER-only                   & 87.043 &88.675 &92.683& 93.032& 90.860& 91.944& N/A& N/A  & N/A&  N/A \\ 
    \midrule
        SPAct~\cite{dave2022spact} & 70.963& 72.548&73.576&74.815 & 73.978& 74.722 & 32.573& 35.921 &40.322& 41.111\\
        Controlled~\cite{Xu_2025_WACV} &77.955& 79.551& 80.153&80.958& 81.021& 81.944 &2.088 & 4.983  &2.220&5.833\\
        Swapping~\cite{xu2022mobilefaceswap} & 65.194& 65.918 & 68.985 & 69.112 & 70.054& 71.111& 4.023& 4.856 &4.194&5.556 \\
        MinusFace~\cite{mi2024minusface} & 76.562 & 78.618 & 82.586& 82.956 & 83.871& 83.611& 6.569 & 8.874 &6.935& 9.166 \\
        K.P.~\cite{xu2026knowledge} & 77.424&79.013 & 81.349 & 82.431& 82.589& 83.215& 2.591& 5.344& 2.541& 5.968 \\
        Pi-Net~\cite{chen2021perceptual} &73.032 & 75.722&78.796 & 80.333&79.065 &80.167 & 1.290& \textbf{3.611} & 1.505 & 4.722\\
        \textbf{TADP} (ours) & \textbf{83.871} & \textbf{82.278}&\textbf{87.634} & \textbf{89.167} & \textbf{88.011}& \textbf{88.889}& \textbf{1.237}& 3.889 &\textbf{1.452} & \textbf{4.444} \\
    \bottomrule
    \end{tabular}
    }

    \label{tab:FER_PLR}
\end{table*}

\noindent\textbf{Temporal Consistency.}
Table~\ref{res:cremad_ravdess_temporal_consistency} compares temporal consistency. TADP achieves the best overall performance on both datasets, with lower flow-guided warping error, lower FVD, lower temporal LPIPS, and higher temporal identity consistency than the competing methods. Overall, TADP produces smoother and more coherent anonymized videos while retaining expression dynamics. A qualitative continuous-sequence visualization is provided in \cref{sec:supp_video_temporal_consistency}.

% \david{Add ablations removing the intra-clip Transformer and inter-clip GRU separately. The present ablation table does not isolate the temporal modules despite temporal coherence being a central claim.}

\begin{table}[!t]
    \centering
    \caption{Temporal consistency on CREMA-D and RAVDESS.}
    \resizebox{0.8\linewidth}{!}{
    \begin{tabular}{c|c|c|c|c|c}
    \toprule
        
    Dataset& Method& FWE $\downarrow$ & FVD $\downarrow$ & tLPIPS $\downarrow$ &  T-ID$\uparrow$ \\

    \midrule
        \multirow{7}{*}{CREMA-D} &SPAct~\cite{dave2022spact} &0.1023 &309.5686 & 0.0974 & 0.1032 \\
        &Controlled~\cite{Xu_2025_WACV}                      &0.1285 &485.5864 & 0.1459 & 0.0481  \\
        &Swapping~\cite{xu2022mobilefaceswap}                &0.0345 &295.5713 & 0.0348 & 0.7691  \\
        &MinusFace~\cite{mi2024minusface}                    &0.0586 &395.5390 & 0.1586 & 0.0287 \\
        &K.P.~\cite{xu2026knowledge}                         & 0.0923&408.3363 & 0.1028 & 0.1785 \\
        &Pi-Net~\cite{chen2021perceptual}                    &0.0358 &292.4851 & 0.0557 & 0.7802 \\
        &\textbf{TADP} (ours)&\textbf{0.0214}&\textbf{284.8561} &\textbf{0.0317} &\textbf{0.7937}  \\ 
        \midrule
        \multirow{7}{*}{RAVDESS} &SPAct~\cite{dave2022spact} &0.0945 &345.5581 & 0.0910  &0.2301 \\
        &Controlled~\cite{Xu_2025_WACV}                      &0.1001 &501.1388 & 0.1368  &0.0732 \\
        &Swapping~\cite{xu2022mobilefaceswap}                &0.0443 &261.5947 & 0.0429  &0.7899 \\
        &MinusFace~\cite{mi2024minusface}                    &0.0434 &419.8853 & 0.1524  &0.0415 \\
        &K.P.~\cite{xu2026knowledge}                         &0.1068 &439.3179 & 0.1145  &0.1067 \\
        &Pi-Net~\cite{chen2021perceptual}                    &0.0284 &251.0348 & 0.0652  &0.7920 \\
        & \textbf{TADP} (ours) & \textbf{0.0203} & \textbf{249.5107} &\textbf{0.0414} &\textbf{0.8032}\\ 

    \bottomrule
    \end{tabular}}
\label{res:cremad_ravdess_temporal_consistency}
\end{table}

The quantitative comparison is complemented by the continuous-sequence visualization below.
\subsubsection{Qualitative Temporal Consistency}\label{sec:supp_video_temporal_consistency}

\cref{fig:supp_video_temporal_consistency} provides a frame-by-frame visualization of a privacy-preserved video generated by TADP. The frames are sampled from a continuous generated sequence and arranged in temporal order from left to right and top to bottom. Across the sequence, the synthesized face preserves a stable anonymized appearance while the mouth, eyes, head pose, and expression-related motion evolve smoothly. This qualitative result complements the quantitative temporal-consistency results above, showing that the proposed temporal transformer and boundary-aware memory help reduce frame-level flickering and maintain coherent expression dynamics throughout the released video.

\begin{figure*}[t]
\centering
\setlength{\tabcolsep}{1pt}
\renewcommand{\arraystretch}{0}
\begin{tabular}{@{}cccccccc@{}}
\includegraphics[width=0.121\textwidth]{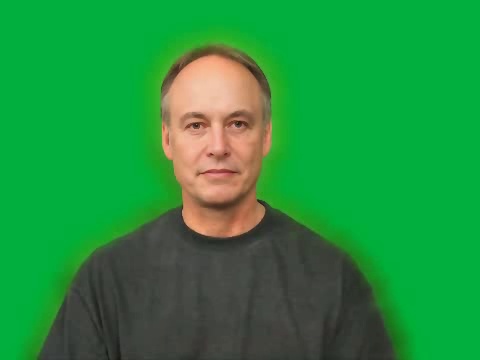} &
\includegraphics[width=0.121\textwidth]{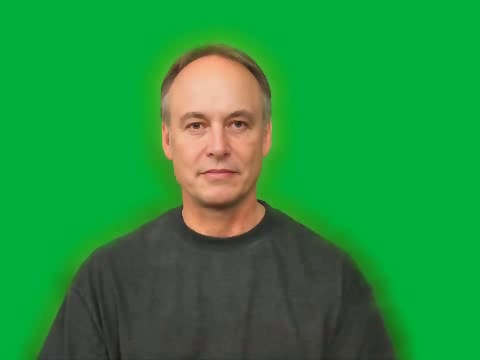} &
\includegraphics[width=0.121\textwidth]{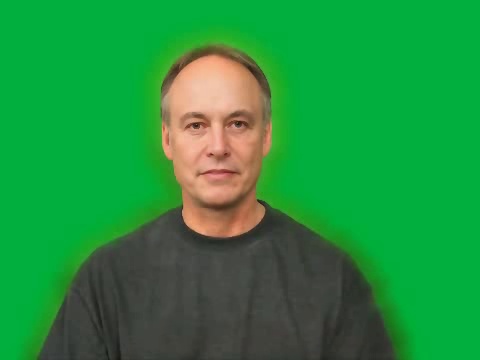} &
\includegraphics[width=0.121\textwidth]{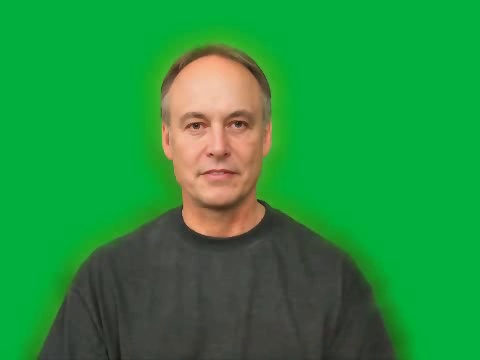} &
\includegraphics[width=0.121\textwidth]{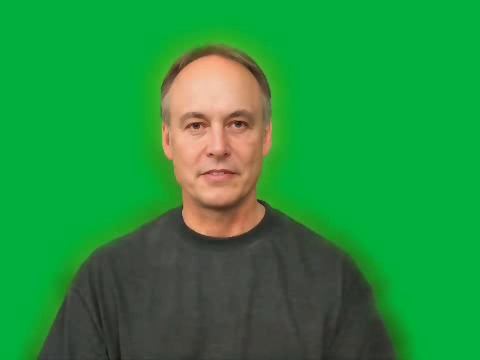} &
\includegraphics[width=0.121\textwidth]{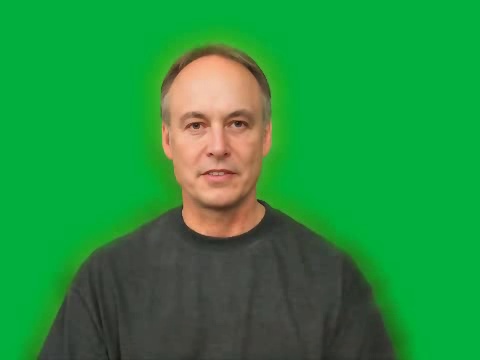} &
\includegraphics[width=0.121\textwidth]{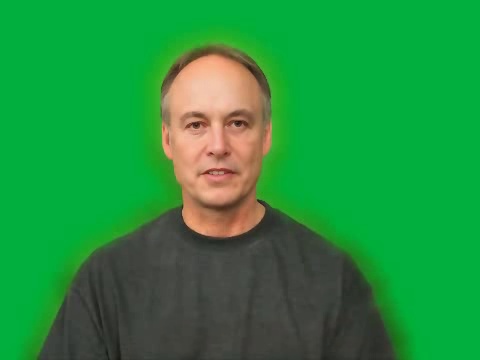} &
\includegraphics[width=0.121\textwidth]{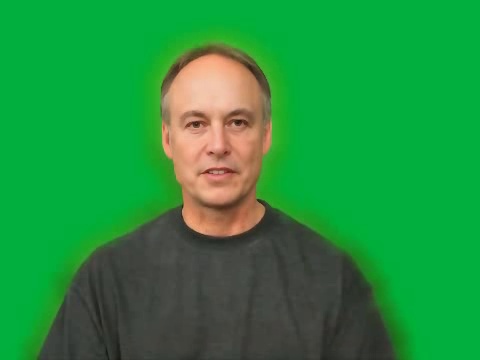} \\
\includegraphics[width=0.121\textwidth]{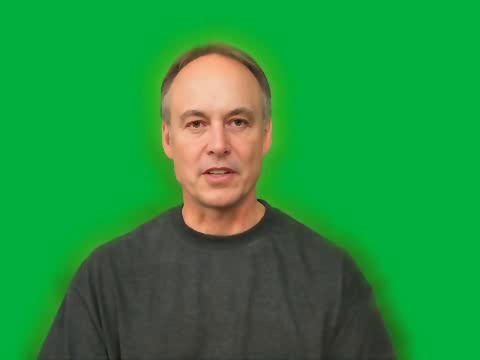} &
\includegraphics[width=0.121\textwidth]{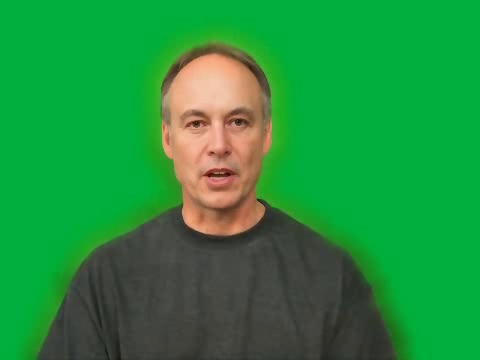} &
\includegraphics[width=0.121\textwidth]{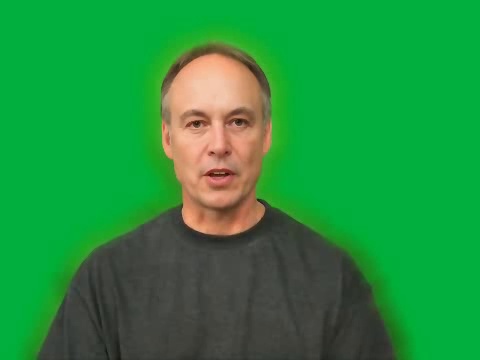} &
\includegraphics[width=0.121\textwidth]{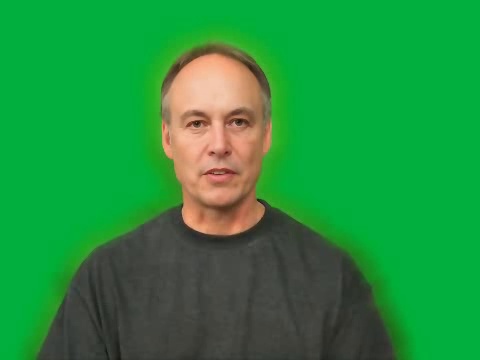} &
\includegraphics[width=0.121\textwidth]{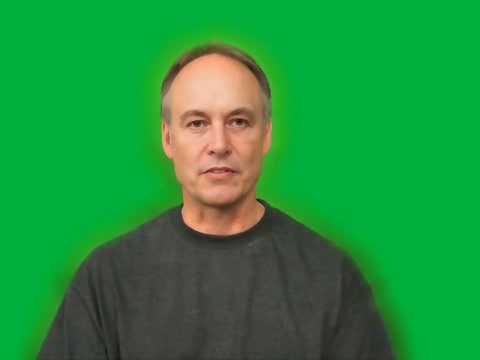} &
\includegraphics[width=0.121\textwidth]{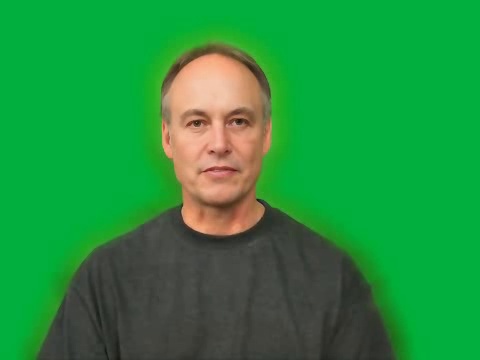} &
\includegraphics[width=0.121\textwidth]{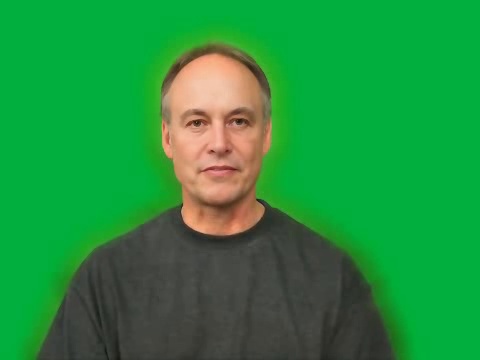} &
\includegraphics[width=0.121\textwidth]{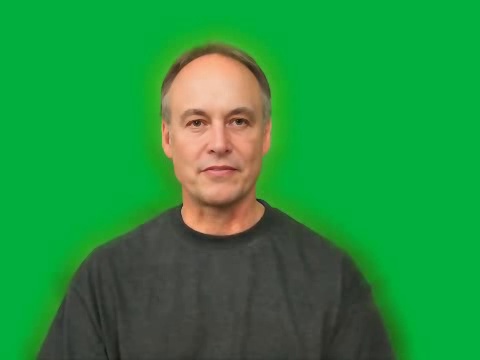} \\
\includegraphics[width=0.121\textwidth]{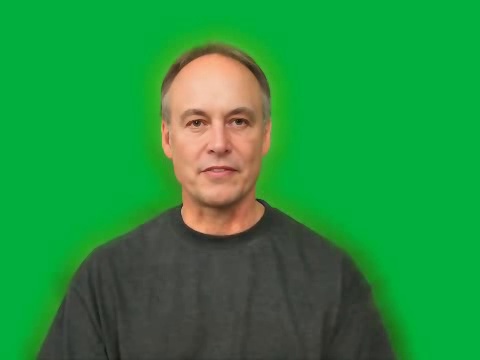} &
\includegraphics[width=0.121\textwidth]{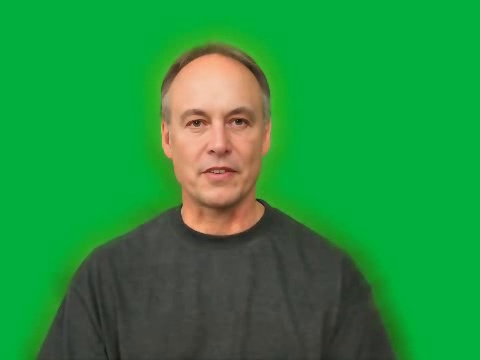} &
\includegraphics[width=0.121\textwidth]{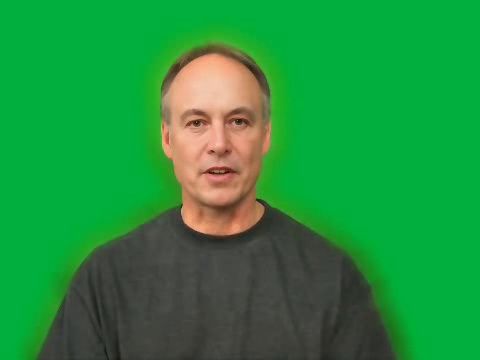} &
\includegraphics[width=0.121\textwidth]{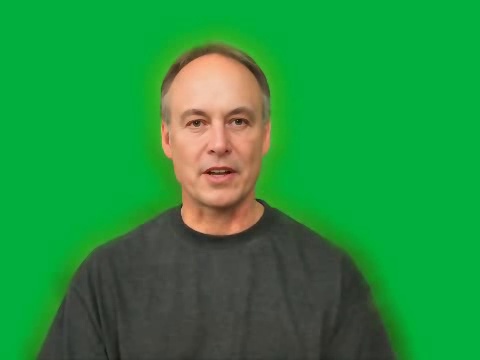} &
\includegraphics[width=0.121\textwidth]{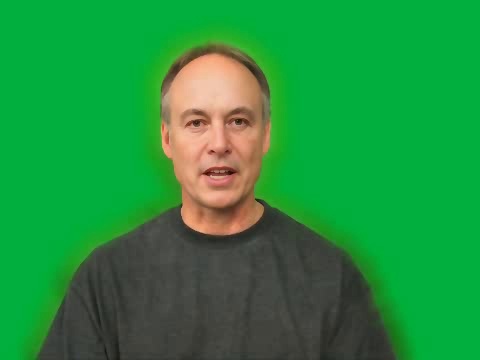} &
\includegraphics[width=0.121\textwidth]{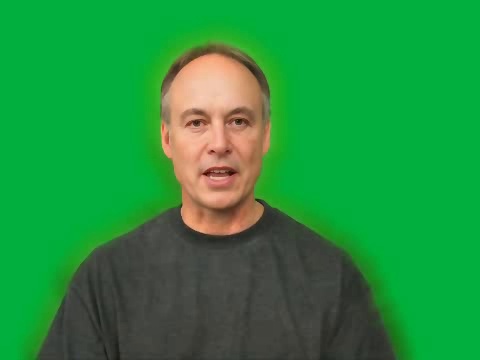} &
\includegraphics[width=0.121\textwidth]{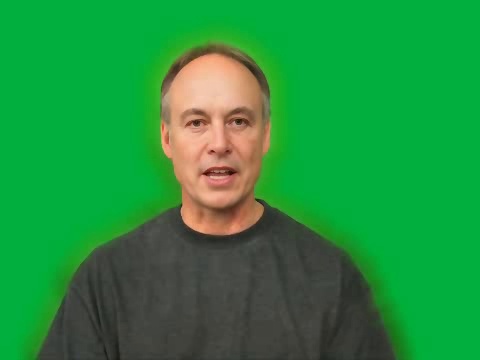} &
\includegraphics[width=0.121\textwidth]{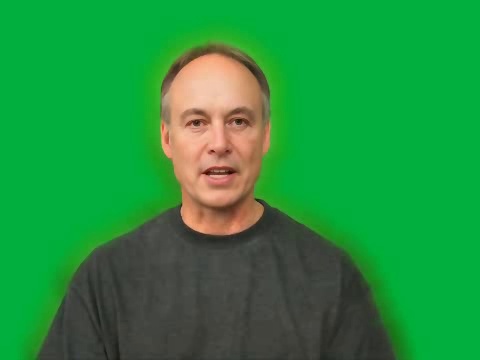} \\
\includegraphics[width=0.121\textwidth]{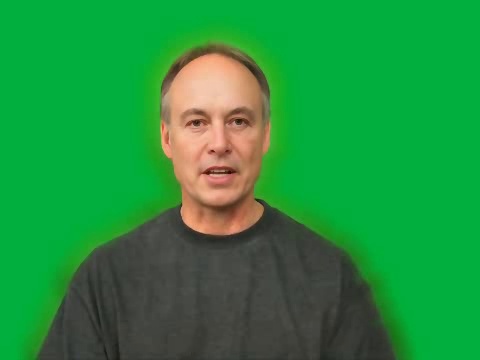} &
\includegraphics[width=0.121\textwidth]{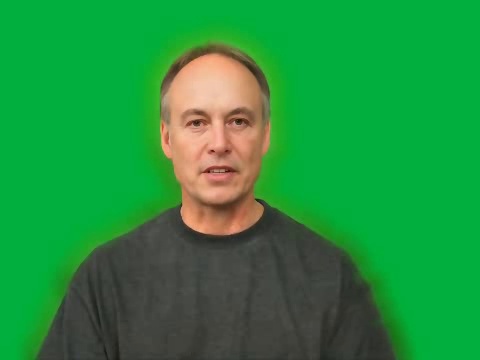} &
\includegraphics[width=0.121\textwidth]{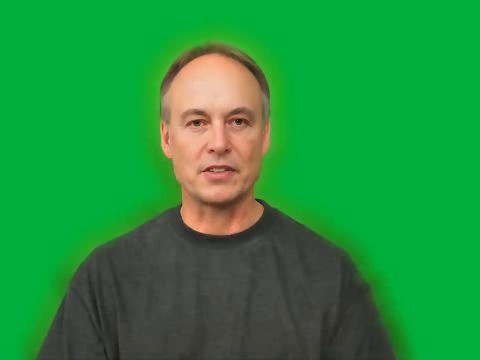} &
\includegraphics[width=0.121\textwidth]{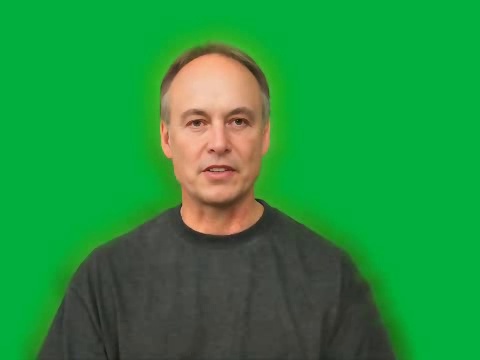} &
\includegraphics[width=0.121\textwidth]{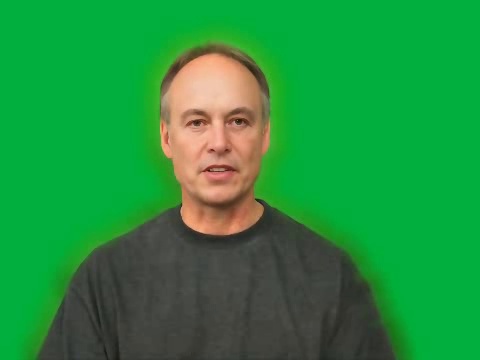} &
\includegraphics[width=0.121\textwidth]{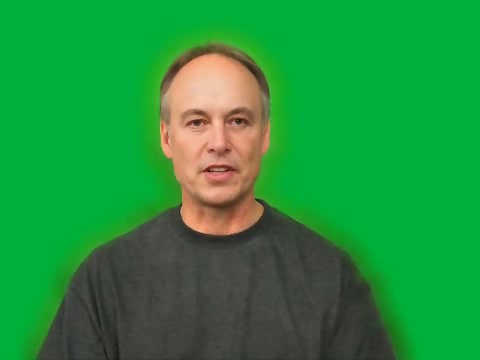} &
\includegraphics[width=0.121\textwidth]{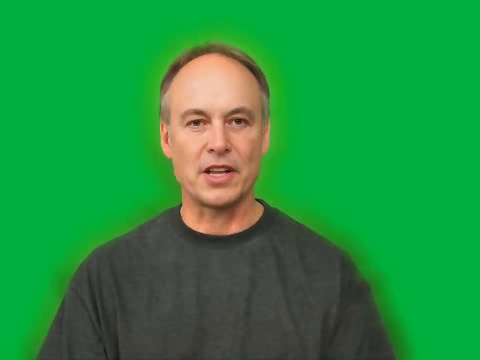} &
\includegraphics[width=0.121\textwidth]{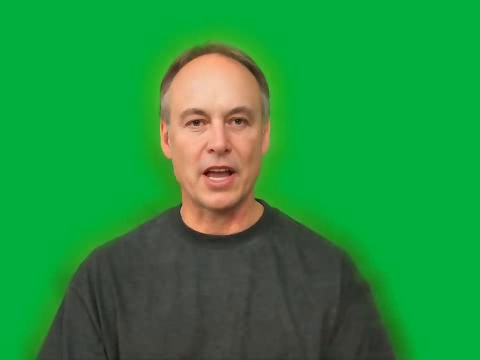} \\
\includegraphics[width=0.121\textwidth]{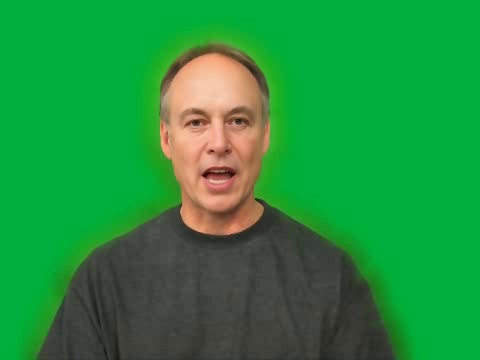} &
\includegraphics[width=0.121\textwidth]{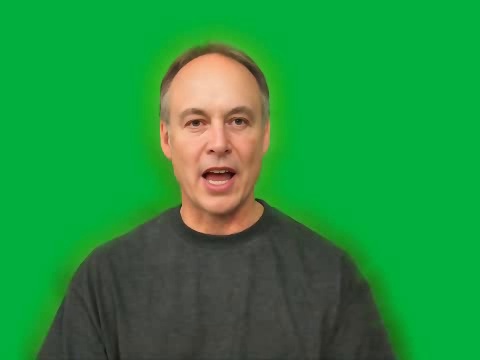} &
\includegraphics[width=0.121\textwidth]{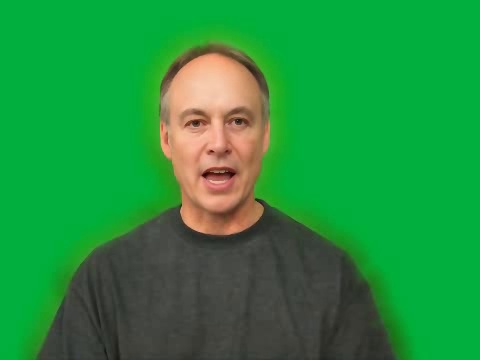} &
\includegraphics[width=0.121\textwidth]{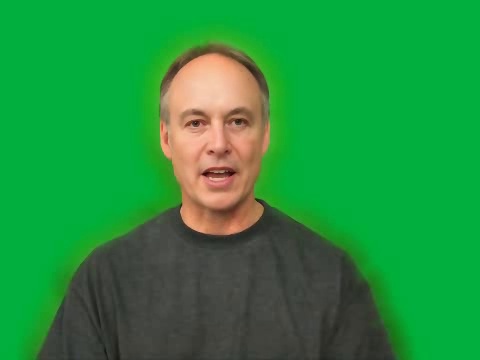} &
\includegraphics[width=0.121\textwidth]{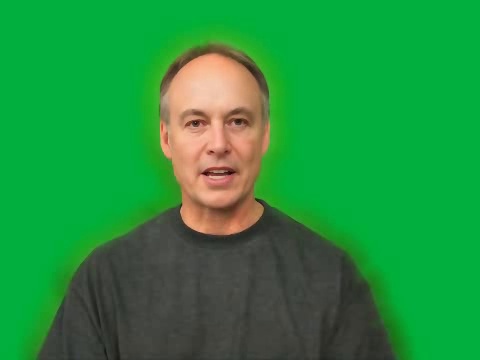} &
\includegraphics[width=0.121\textwidth]{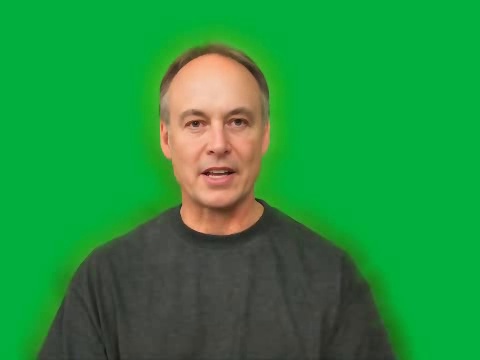} &
\includegraphics[width=0.121\textwidth]{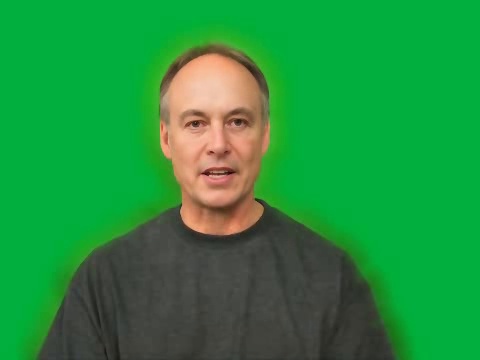} &
\includegraphics[width=0.121\textwidth]{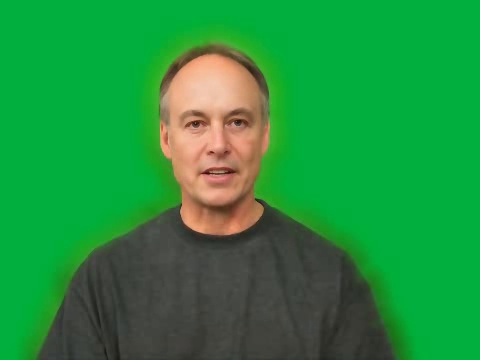} \\
\includegraphics[width=0.121\textwidth]{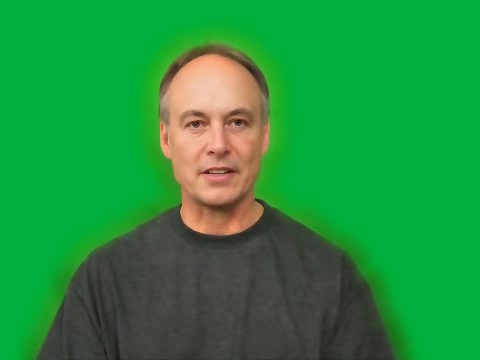} &
\includegraphics[width=0.121\textwidth]{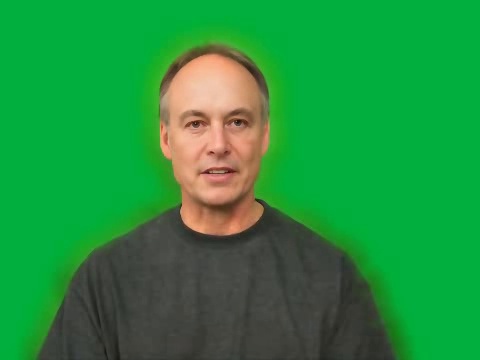} &
\includegraphics[width=0.121\textwidth]{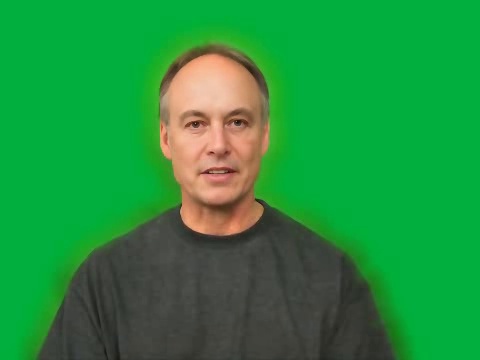} &
\includegraphics[width=0.121\textwidth]{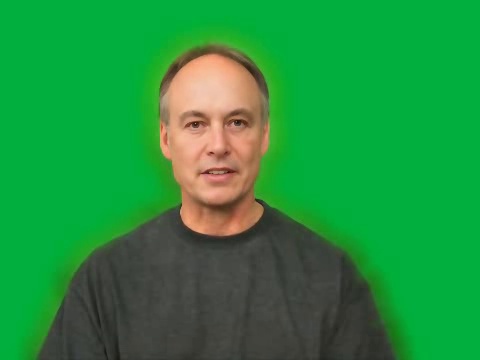} &
\includegraphics[width=0.121\textwidth]{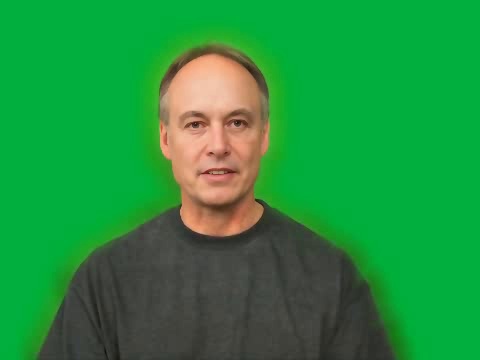} &
\includegraphics[width=0.121\textwidth]{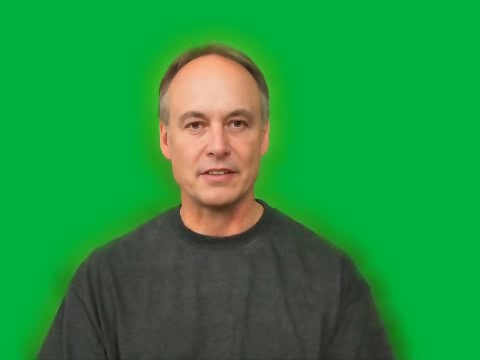} &
\includegraphics[width=0.121\textwidth]{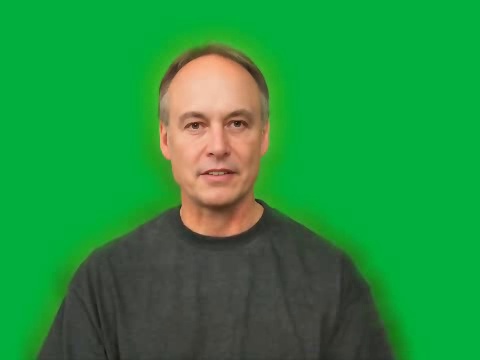} &
\includegraphics[width=0.121\textwidth]{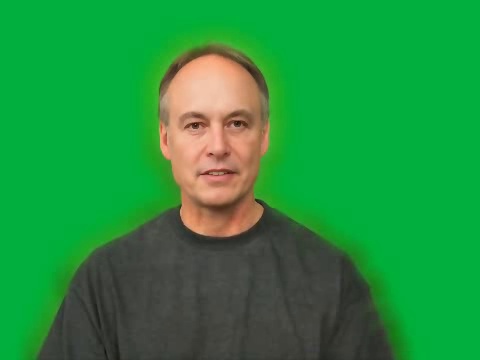} \\
\includegraphics[width=0.121\textwidth]{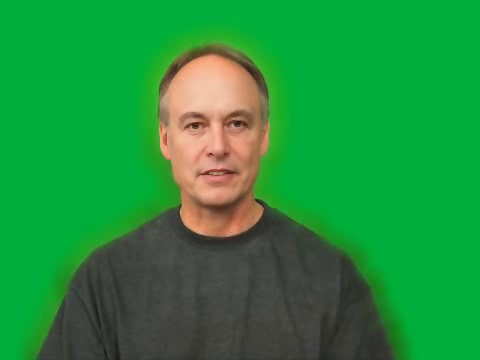} &
\includegraphics[width=0.121\textwidth]{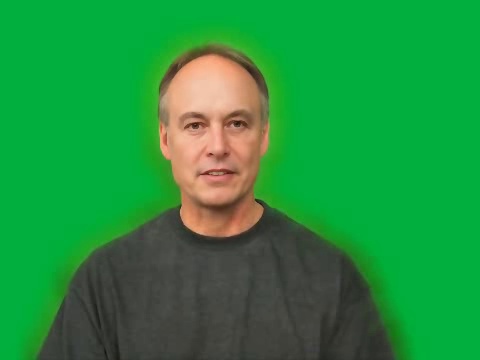} &
\includegraphics[width=0.121\textwidth]{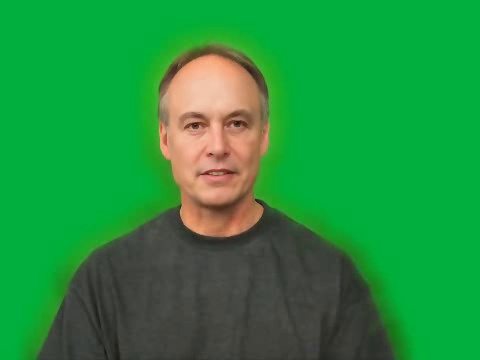} &
\includegraphics[width=0.121\textwidth]{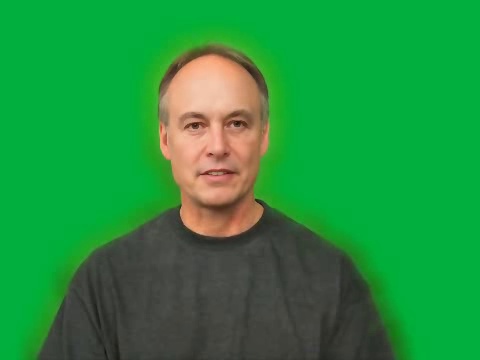} &
\includegraphics[width=0.121\textwidth]{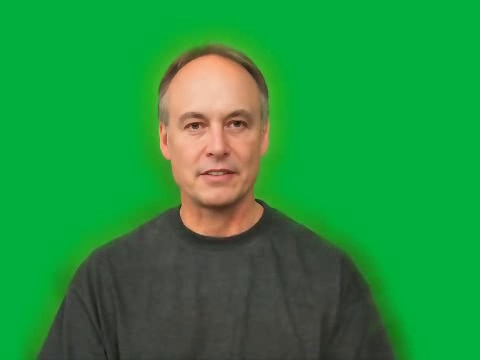} &
\includegraphics[width=0.121\textwidth]{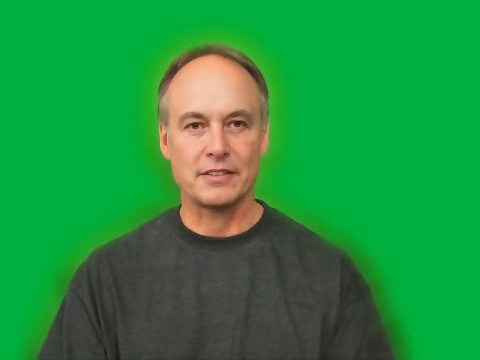} &
\includegraphics[width=0.121\textwidth]{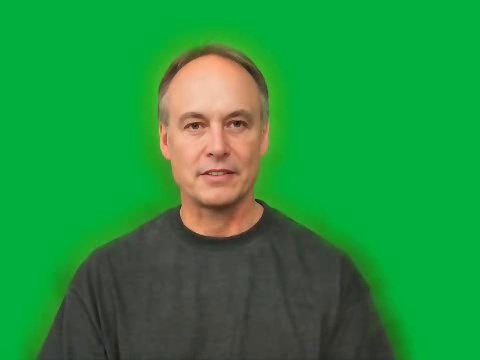} &
\includegraphics[width=0.121\textwidth]{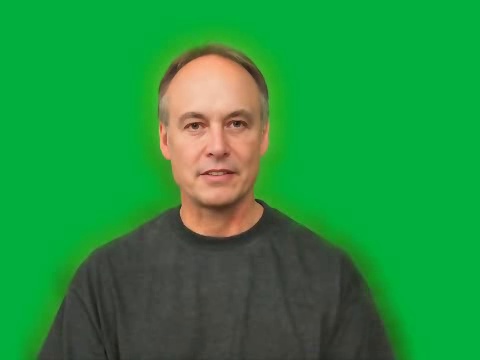} \\
\includegraphics[width=0.121\textwidth]{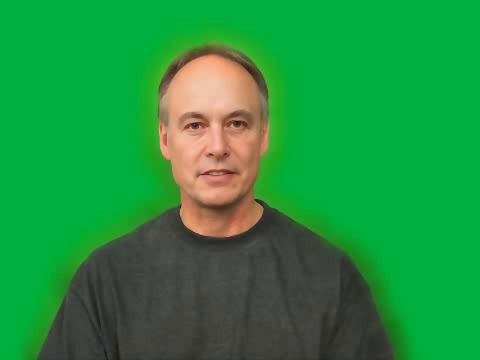} &
\includegraphics[width=0.121\textwidth]{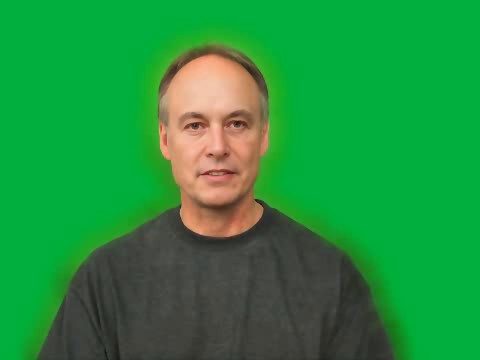} &
\includegraphics[width=0.121\textwidth]{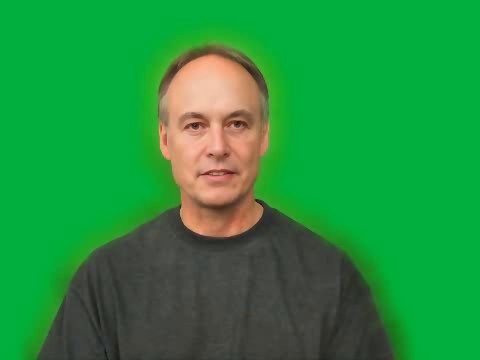} &
\includegraphics[width=0.121\textwidth]{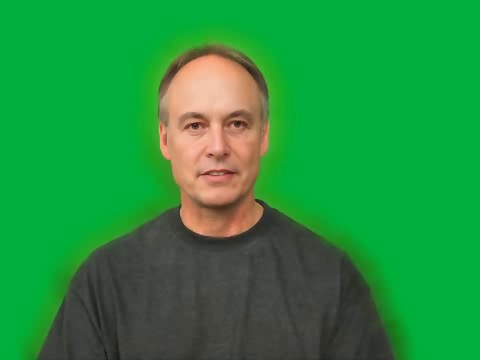} &
\includegraphics[width=0.121\textwidth]{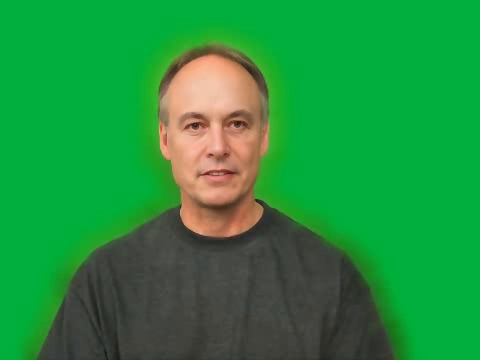} &
\includegraphics[width=0.121\textwidth]{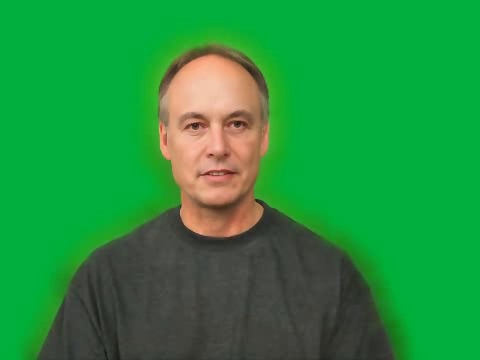} &
\includegraphics[width=0.121\textwidth]{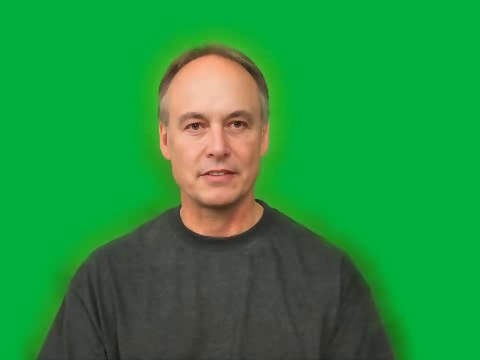} &
\includegraphics[width=0.121\textwidth]{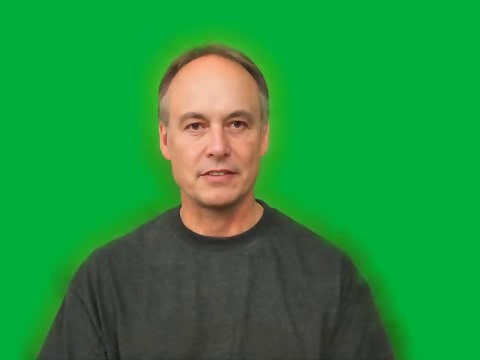}
\end{tabular}
\caption{Video temporal consistency visualization. The 64 generated frames are arranged in temporal order from left to right and top to bottom, showing that the privacy-preserved video maintains smooth frame-to-frame appearance and expression dynamics.}
\label{fig:supp_video_temporal_consistency}
\end{figure*}

% \david{Table~5 and Fig.~3 remain weaker than the main experimental story. FID/MMD in ArcFace space and t-SNE separation are descriptive evidence, not direct proof of anonymity. Consider moving them to the supplementary material and using the recovered main-paper space for the requested qualitative comparison or a privacy--utility plot.}

% \david{If Fig.~3 is retained, soften ``confirm'' to ``is consistent with'' or ``supports''. The result shows that features remain separated under one ArcFace representation and the evaluated recovery attack, not that identity cannot be recovered in general.}

\noindent\textbf{Recovery Attack.}
We evaluate robustness against anonymization-aware, paired-supervision recovery attacks. In addition to the frame-based U-Net attacker~\cite{ronneberger2015u}, we employ RVRT~\cite{liang2022recurrent} as a stronger video recovery attacker that jointly exploits spatial and temporal information from each 16-frame clip. For each anonymization method, both attackers are separately trained to reconstruct the corresponding original clips from that method's privacy-preserved outputs. 
We measure visual recovery using SSIM and PSNR and evaluate identity leakage after recovery using PLR and cPLR. Lower values indicate stronger resistance to the evaluated recovery attacks. The results are shown in Table~\ref{tab:privacy_leakage}. TADP achieves competitive resistance across both datasets.

\begin{table}[!ht]
    \centering
    \caption{After U-Net and RVRT recovery attacks, evaluating on four metrics.}
    \resizebox{0.99\linewidth}{!}{
    \begin{tabular}{c|c|c|c|c|c|c|c|c|c}
    \toprule
        &\multirow{2}{*}{Method}& \multicolumn{4} {c|}{CREMA-D}& \multicolumn{4} {c}{RAVDESS}\\ \cline{3-10}
        & & SSIM $\downarrow$& PSNR$\downarrow$&PLR $\downarrow$&cPLR $\downarrow$&SSIM$\downarrow$& PSNR$\downarrow$&PLR$\downarrow$&cPLR $\downarrow$ \\
    \midrule
    \multirow{7}{*}{U-Net}    
        &SPAct~\cite{dave2022spact}           & 0.5358& 23.1459& 42.890\%& 54.913\%& 0.5421 &23.0751&43.861\% & 56.497\% \\
        &Controlled~\cite{Xu_2025_WACV}       & 0.4283& 11.8560& 5.847\% & 10.573\%& 0.4641 &11.9177& 6.600\% & 11.556\% \\
        &Swapping~\cite{xu2022mobilefaceswap} & 0.5012& 17.9647& 38.950\%& 42.321\%& 0.5186 &18.0398&39.315\% & 44.789\% \\
        &MinusFace~\cite{mi2024minusface}     & 0.4591& 13.5457& 97.395\%& 98.667\%& 0.4461 &13.5287&97.477\% & 98.991\%\\
        &K.P.~\cite{xu2026knowledge}          & 0.3959& 11.8308& 6.103\%& 9.425\%& 0.4512& 10.314 &6.813\%&10.543\% \\
        &Pi-Net~\cite{chen2021perceptual}     & 0.3901& 11.0414& \textbf{1.942\%}& 2.167\% & 0.4118 & \textbf{9.7944} & 3.889\% & 4.694\%\\
        &\textbf{TADP} (ours)                 &\textbf{0.3754} & \textbf{10.8593}&1.989\% &\textbf{2.012\%} & \textbf{0.4002} &9.8573&\textbf{3.611\%} & \textbf{3.976\%}\\
    \midrule
    \multirow{7}{*}{RVRT} 
        &SPAct~\cite{dave2022spact}           & 0.6948& 29.6454 & 50.587\%& 54.941\%& 0.7034& 30.0396& 51.581\% & 57.389\% \\
        &Controlled~\cite{Xu_2025_WACV}       & 0.5247& 14.4512 & 8.935\%& 10.243\%& 0.5889& 15.5261& 9.628\% & 11.094\% \\
        &Swapping~\cite{xu2022mobilefaceswap} & 0.5831& 21.6778 & 42.601\%& 56.325\%& 0.5991& 26.1709& 41.107\% & 52.139\% \\
        &MinusFace~\cite{mi2024minusface}     & 0.5190& 18.4587 & 98.326\%& 99.417\%& 0.5326& 20.1276& 98.754\% & 99.124\%\\
        &K.P.~\cite{xu2026knowledge}          & 0.5421& 14.0121 & 7.347\%& 9.439\%& 0.5017& 11.5491& 7.046\%& 10.435\% \\
        &Pi-Net~\cite{chen2021perceptual}     & 0.4003& 11.0985 & 2.425\%& 3.146\%& 0.4545& 13.6399& \textbf{3.261}\% & 4.847\%\\
        &\textbf{TADP} (ours)                 &\textbf{0.3996} & \textbf{11.0348}&\textbf{2.106\%} &\textbf{2.912\%} & \textbf{0.4353} & \textbf{13.3537}&3.742\% & \textbf{3.498\%}\\
    \bottomrule
    \end{tabular}}
    \label{tab:privacy_leakage}

\end{table}

\noindent\textbf{Relation to Formal Privacy.}
PLR, cPLR, and recovery attacks measure operational identity leakage, whereas the reported $(\varepsilon,\delta)$ guarantee bounds pairwise output distinguishability under 16-frame clip-replacement adjacency. These quantities are not directly interchangeable, and low empirical leakage does not estimate or validate the accounted $\varepsilon$. \Cref{sec:supp_dp_audit} formalizes the corresponding hypothesis-testing audit and its interpretation.

\noindent\textbf{Open-set Identity Leakage.}
The closed-set and recovery results are complemented by open-set identification and verification under attackers adapted only on development actors.
\begin{table*}[t]
    \centering
    \caption{Open-set identity verification and identification under anonymization-aware attackers. For each anonymization method, the identity attacker is separately adapted using only privacy-preserved outputs from the development actors. All evaluation actors are unseen during attacker adaptation. ``Direct'' denotes privacy-preserved probes, and ``Recovered'' denotes probes reconstructed by the method-specific RVRT attacker. Lower TAR and DIR and higher EER indicate stronger identity privacy.}
    \label{tab:open_set_identity}
    \resizebox{\textwidth}{!}{
        \begin{tabular}{ll|ccc|cc|ccc|cc}
            \toprule
            \multirow{2}{*}{Method} & \multirow{2}{*}{Probe} & \multicolumn{5}{c|}{CREMA-D} & \multicolumn{5}{c}{RAVDESS} \\ \cline{3-12}
            & & TAR@FAR=$10^{-2}$ $\downarrow$ & TAR@FAR=$10^{-3}$ $\downarrow$ & EER $\uparrow$ & DIR@FAR=$10^{-2}$ $\downarrow$ & DIR@FAR=$10^{-3}$ $\downarrow$ & TAR@FAR=$10^{-2}$ $\downarrow$ & TAR@FAR=$10^{-3}$ $\downarrow$ & EER $\uparrow$ & DIR@FAR=$10^{-2}$ $\downarrow$ & DIR@FAR=$10^{-3}$ $\downarrow$ \\
            \midrule
            Original & Direct & 94.263 & 88.136 & 3.950 & $90.359\pm1.105$ & $82.490\pm1.542$ & 96.058 & 91.227 & 2.836 & $93.512\pm0.913$ & $87.126\pm1.362$ \\
            \midrule

            SPAct & Direct & 38.670 & 22.448 & 28.743 & $31.592\pm2.431$ & $16.884\pm1.920$ & 41.869 & 25.797 & 26.540 & $34.644\pm3.153$ & $19.725\pm2.539$ \\ 
            & Recovered & 52.720 & 34.981 & 22.114 & $45.231\pm2.701$ & $28.164\pm2.277$ & 57.430 & 39.607 & 19.875 & $49.336\pm3.017$ & $32.516\pm2.740$ \\

            Controlled & Direct & 6.405 & 2.127 & 42.810 & $4.953\pm0.830$ & $1.573\pm0.418$ & 8.213 & 2.877 & 40.622 & $6.549\pm1.193$ & $2.094\pm0.650$ \\ 
            & Recovered & 12.871 & 5.310 & 37.615 & $10.628\pm1.373$ & $3.913\pm0.704$ & 15.113 & 6.425 & 35.217 & $12.716\pm1.644$ & $4.819\pm0.983$ \\

            Swapping & Direct & 8.118 & 2.943 & 41.027 & $6.543\pm1.020$ & $2.011\pm0.524$ & 9.459 & 3.526 & 39.830 & $7.865\pm1.218$ & $2.617\pm0.715$ \\ 
            & Recovered & 48.530 & 30.754 & 24.023 & $40.231\pm2.527$ & $23.605\pm2.076$ & 52.314 & 34.126 & 22.613 & $44.134\pm2.916$ & $26.839\pm2.448$ \\

            MinusFace & Direct & 11.263 & 4.672 & 38.94 & $9.057\pm1.213$ & $3.131\pm0.738$ & 13.016 & 5.524 & 37.435 & $10.856\pm1.506$ & $4.063\pm0.961$ \\ 
            & Recovered & 91.018 & 81.893 & 6.212 & $87.259\pm1.426$ & $75.624\pm1.969$ & 93.276 & 85.015 & 5.146 & $89.822\pm1.317$ & $79.498\pm1.811$ \\

            K.P.~\cite{xu2026knowledge} & Direct &6.443 & 2.462 & 43.859& $4.507\pm1.034$ & $1.505\pm0.825$& 8.342& 3.114 & 42.011 & $5.978\pm0.617$& $1.923\pm0.854$\\
            & Recovered & 14.423 & 5.905 & 39.872& $11.018\pm1.665$ &$4.006\pm1.407$ & 13.961 & 7.865 & 36.542 & $12.859\pm2.045$& $5.266\pm1.953$ \\

            Pi-Net & Direct & 3.231 & 0.830 & 46.318 & $2.304\pm0.549$ & $0.507\pm0.278$ & 3.911 & 1.055 & 45.116 & $3.176\pm0.744$ & $0.867\pm0.362$ \\ 
            & Recovered & 4.871 & 1.456 & 44.853 & $3.519\pm0.668$ & $0.948\pm0.314$ & 6.298 & 1.983 & 42.924 & $4.981\pm0.996$ & $1.367\pm0.449$ \\
            \midrule

            \textbf{TADP} & Direct & \textbf{2.745} & \textbf{0.637} & \textbf{47.174} & $\mathbf{1.959\pm0.401}$ & $\mathbf{0.417\pm0.219}$ & 4.110 & 1.016 & \textbf{45.338} & $\mathbf{3.049\pm0.667}$ & $\mathbf{0.769\pm0.286}$ \\ 
            & Recovered & \textbf{3.460} & \textbf{0.973} & \textbf{46.247} & $\mathbf{2.546\pm0.532}$ & $\mathbf{0.647\pm0.272}$ & \textbf{5.091} & \textbf{1.461} & \textbf{44.156} & $\mathbf{3.813\pm0.747}$ & $\mathbf{0.942\pm0.372}$ \\
            \bottomrule
            \end{tabular}
    }
\end{table*}

The direct-probe rows quantify identity leakage from the released privacy-preserved clips, while the recovered-probe rows test whether a method-specific temporal recovery attacker increases open-set biometric leakage. The original-video row provides a non-private reference for attacker strength. We interpret verification and identification jointly because a low closed-set recognition rate alone does not establish the ability to reject identities outside the enrolled gallery.

As shown in~\cref{tab:open_set_identity}, TADP provides the strongest overall open-set identity protection among the evaluated methods. On CREMA-D, it yields the lowest TAR and DIR and the highest EER for both direct and recovered probes. On RAVDESS, TADP also achieves the lowest DIR and highest EER for direct probes, although Pi-Net has slightly lower direct-probe TAR; after RVRT recovery, TADP obtains the best results across all reported verification and identification metrics. Moreover, the relatively small change from the direct to recovered TADP rows indicates that the evaluated temporal recovery attacker reveals substantially less additional identity information for TADP than for methods such as Swapping and MinusFace. These results are empirical evidence under the specified attackers and do not replace the formal clip-level DP guarantee.

\subsubsection{Relation Between Empirical Leakage and Formal DP}\label{sec:supp_dp_audit}
The formal and empirical privacy results answer different questions. The clip-level $(\varepsilon,\delta)$ guarantee bounds the distinguishability of the complete output distributions induced by any two clips satisfying the replacement-adjacency definition. In contrast, PLR, cPLR, open-set verification, and recovery attacks measure identity information extracted by specific adversaries under the evaluated data distribution. Consequently, PLR or cPLR cannot be converted directly into an empirical $\varepsilon$, and low attack accuracy does not establish that the theoretical guarantee is tight.

A statistically calibrated empirical audit must instead follow the same adjacency relation as the formal analysis. For an adjacent pair $(X,X')$, the mechanism is run independently multiple times to obtain samples from $\mathcal M(X)$ and $\mathcal M(X')$. An anonymization-aware binary video discriminator is trained on one subset of these releases and evaluated on a disjoint subset. For the event $S_t$ induced by a decision threshold $t$, differential privacy requires
\begin{align}
\Pr[\mathcal M(X)\in S_t]
\leq
e^{\varepsilon}\Pr[\mathcal M(X')\in S_t]+\delta.
\end{align}
Let $p_t$ and $q_t$ denote the true-positive and false-positive probabilities for this event. A finite-sample, one-sided empirical lower bound is
\begin{align}
\varepsilon_{\mathrm{audit}}^{LB}
=\max_t\left\{0,
\log\frac{p_t^{LB}-\delta}{q_t^{UB}}
\right\},
\label{eq:empirical_dp_audit}
\end{align}
whenever $p_t^{LB}>\delta$, where $p_t^{LB}$ and $q_t^{UB}$ are lower and upper Clopper--Pearson confidence limits. The audit must be repeated after exchanging $X$ and $X'$ and for complementary decision events; the maximum across thresholds, directions, and evaluated adjacent pairs is reported. Candidate pairs should include different-identity/same-expression, same-identity/different-expression, different-identity/different-expression, randomly sampled pairs, and pairs with large clipped-encoder distance.

The resulting $\varepsilon_{\mathrm{audit}}^{LB}$ is an attack-dependent empirical lower bound, whereas the accountant provides a mechanism-wide upper bound. If a statistically valid audited lower bound exceeded the accounted value, this would indicate a possible implementation or accounting error. Conversely, an audited lower bound below the accountant value would only show that the evaluated discriminator found no violation; it would not prove implementation correctness, tightness of the theoretical bound, or robustness against every possible distinguisher.

The present PLR, cPLR, open-set, and recovery experiments were not generated from repeated independent releases of fixed adjacent clip pairs and therefore do not satisfy this calibration protocol. We accordingly report them only as operational identity-leakage evidence and do not reinterpret them as empirical estimates of $\varepsilon$. Performing the repeated-release hypothesis-testing audit in~\cref{eq:empirical_dp_audit} is an important direction for a separate implementation-level privacy audit.

\begin{table}[t]
    \centering
    \caption{Comparison of identity feature distributions among original data (O), privacy-preserved data (P) and recovery-attacked data (R) under U-Net and RVRT.}        \resizebox{0.85\linewidth}{!}{
         \begin{tabular}{c|c|ccc|ccc}
      \toprule
      &\multirow{2}{*}{Dataset} & \multicolumn{3} {c|}{FID} & \multicolumn{3} {c}{MMD} \\ \cline{3-8}
      &   & O vs P & O vs R & P vs R &O vs P & O vs R & P vs R  \\
      \midrule
       \multirow{2}{*}{U-Net}& CREMA-D 
       & 130.5908 & 144.9318 & 72.0473 & 0.0308 & 0.0374 & 0.0014\\
       & RAVDESS 
       & 412.3707 & 436.4749 & 356.8630 & 0.0331 & 0.0427 & 0.0057\\
        \midrule
       \multirow{2}{*}{RVRT}& CREMA-D 
       & 133.6474 & 164.9867 & 81.7367 & 0.0325 & 0.0395 & 0.0021\\
       & RAVDESS 
       & 454.1176 & 439.0419 & 350.8485 & 0.0510 & 0.0428 & 0.0052\\
        \bottomrule
      \end{tabular}
        }
    \label{tab:fid_mmd}
\end{table}

\noindent\textbf{Empirical Insights.}
We use Table~\ref{tab:fid_mmd} as descriptive quantitative evidence in ArcFace feature space, rather than as direct proof of anonymity. 
ArcFace~\cite{deng2019arcface} is used to extract identity features, and FID~\cite{heusel2017gans} and MMD measure distributional discrepancy among original data (O), privacy-preserved data (P), and recovery-attacked (R) data.
As shown in~\Cref{tab:fid_mmd}, TADP shifts privacy-preserved identity features away from the original distribution, and the evaluated recovery attacks do not move them back toward the original identity space. Instead, the recovered samples remain closer to the privacy-preserved distribution, suggesting that the induced ArcFace-space identity shift is stable under the evaluated attacks.
The corresponding t-SNE visualizations in \cref{fig:tsne,fig:tsne-rvrt} are consistent with these quantitative results: privacy-preserved and recovery-attacked identity features remain separated from the original-feature distribution under the evaluated ArcFace representation. 
Together, these results support, but do not by themselves prove, that TADP maintains an identity-feature distribution shift under the evaluated recovery attacks.

An in-the-wild cross-domain stress test on DFEW~\cite{jiang2020dfew} reveals substantial performance degradation under diverse distribution shifts, clarifying the method's current scope; complete results and diagnostics are provided in \cref{in-the-wild_data}.

Having established the primary privacy--utility results, we now examine where expression information remains accessible, how stable the mechanism is across noise draws, and how the method behaves under domain shift.

\subsection{Diagnostic and Extended Results}\label{more_results}

\subsubsection{Identity-feature Data Distribution}
\begin{figure}[!ht]
    \centering
    \includegraphics[width=0.49\linewidth]{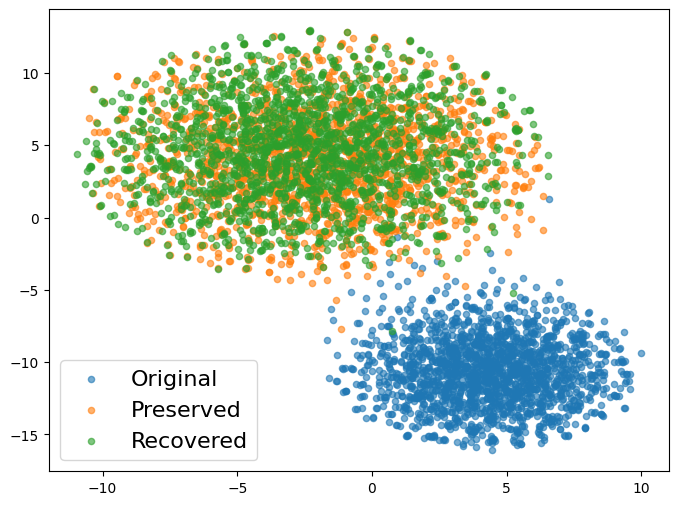}
    \includegraphics[width=0.49\linewidth]{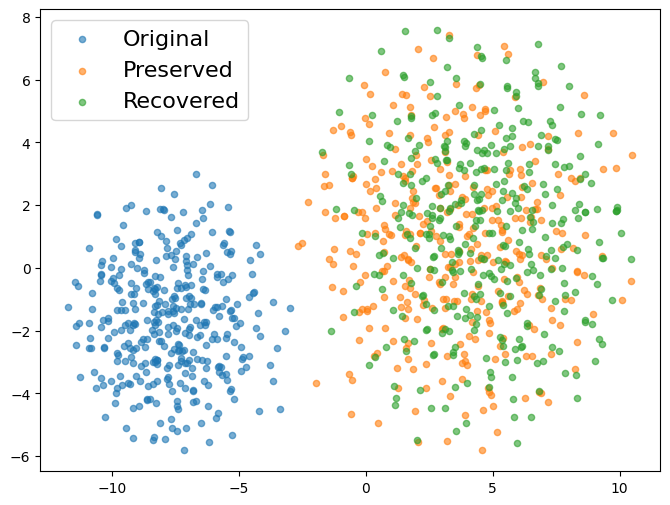}
    \caption{t-SNE visualization of identity features on CREMA-D (left) and RAVDESS (right). Blue, orange, and green points denote original, privacy-preserved, and recovery-attacked identity features (U-Net), respectively.}
    \label{fig:tsne}
\end{figure}

\cref{fig:tsne} visualizes the ArcFace identity-feature distributions before and after the evaluated U-Net recovery attack. On both CREMA-D and RAVDESS, the privacy-preserved features are clearly shifted away from the original features. The U-Net-recovered features substantially overlap with the privacy-preserved features and remain separated from the original-feature cluster, indicating that this attacker does not move the released samples back toward the original identity distribution in the evaluated ArcFace representation. This observation is consistent with the FID/MMD distances and the post-recovery PLR and cPLR results reported above. Nevertheless, t-SNE is a projection-dependent qualitative visualization; it supports the observed distribution shift under this particular attacker and feature extractor, but does not by itself prove anonymity or robustness against arbitrary recovery attacks.

\begin{figure}[!ht]
    \centering
    \includegraphics[width=0.49\linewidth]{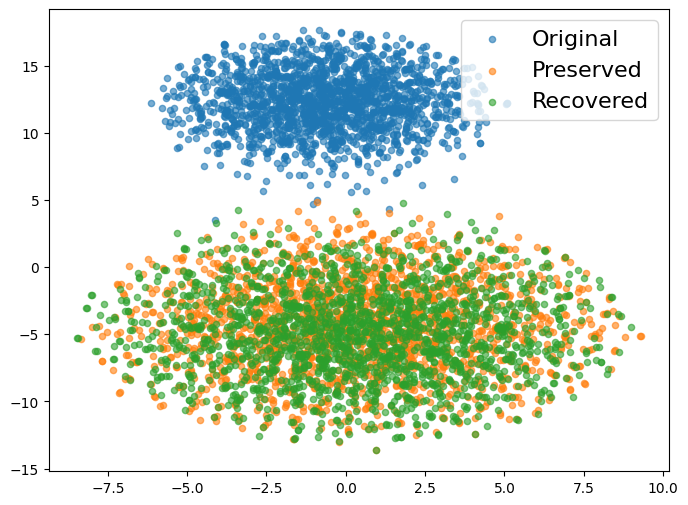}
    \includegraphics[width=0.49\linewidth]{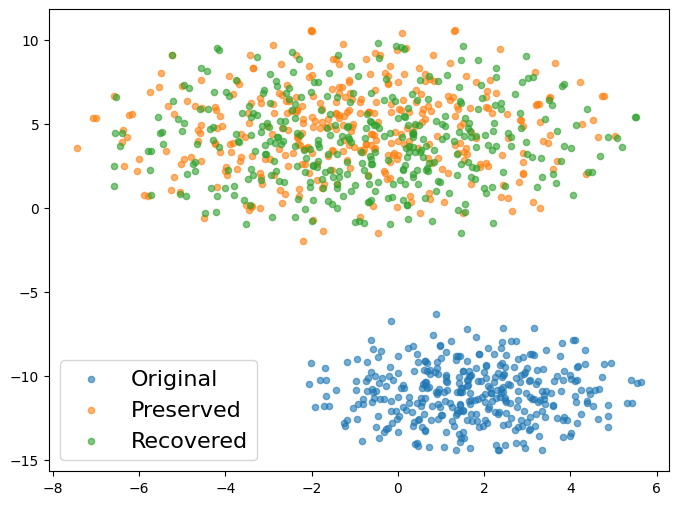}
    \caption{t-SNE visualization of identity features on CREMA-D (left) and RAVDESS (right). Blue, orange, and green points denote original, privacy-preserved, and recovery-attacked identity features (RVRT), respectively.}
    \label{fig:tsne-rvrt}
\end{figure}

\cref{fig:tsne-rvrt} presents the corresponding analysis for RVRT, which can exploit both spatial and temporal information within a privacy-preserved clip. Despite this stronger recovery capability, the RVRT-recovered identity features remain highly intermixed with the privacy-preserved features on both datasets and stay clearly separated from the original-feature distribution. In particular, no visible shift of the recovered cluster toward the original cluster is observed in either CREMA-D or RAVDESS. The qualitative pattern is therefore consistent with that of the frame-based U-Net attacker and with the low post-RVRT PLR and cPLR values in the preceding results, suggesting that temporal recovery does not restore the original ArcFace identity distribution under the evaluated attack protocol. As with the U-Net visualization, this t-SNE result is interpreted only as descriptive evidence and should be considered jointly with the quantitative recovery and identity-leakage metrics.

\subsubsection{Stage-wise Expression Information and Noise Diagnostics}\label{stage_wise}
To explore how expression-related utility is retained under strong Gaussian perturbation, we conduct a stage-wise diagnostic of the representations produced by TADP. We extract features at four locations: (1) the clipped expression representation before privatization, (2) immediately after Gaussian privatization, (3) after soft quantization, and (4) after temporal contextualization. For each latent representation, we freeze TADP, apply temporal average pooling, and train the same linear classifier as an FER probe on the training portion of the actor-disjoint development split. The probes are evaluated on the actor-disjoint test split. We use AdamW with learning rate $10^{-3}$, weight decay $10^{-4}$, cross-entropy loss, a maximum of 50 epochs, and early-stopping patience of 5. Each probe experiment is repeated using seeds 42, 43, and 44.

\begin{table*}[!ht]
    \centering
    \caption{Stage-wise expression information. TADP is frozen and FER is evaluated using temporal mean pooling followed by one linear layer. The released videos are evaluated using R3D,  TimeSformer and VideoMAE V2 trained on the corresponding anonymized training videos. 
    Results are mean$\pm$std over three seeds 42, 43, and 44. 
    The video-level R3D, TimeSformer and VideoMAE V2 entries are single-run results reproduced from the primary evaluation and are shown only as downstream-utility references; they are not directly comparable to the latent-probe accuracies.
    The pre-privatization representation is included as a non-private latent reference and is never released.}
    \label{tab:stagewise_fer}
    \resizebox{0.7\textwidth}{!}{
        \begin{tabular}{lllcc}
        \toprule
        Stage & Evaluated representation & Evaluator & CREMA-D (\%) & RAVDESS (\%) \\
        \midrule
        Random chance & -- & -- & 16.67 & 12.50 \\ 
        \midrule
        Pre-privatization & Clipped expression feature $\bar z^{expr}$ & Linear probe   & $62.03\pm 0.15\%$ & $68.05\pm 0.45\%$ \\
        DP release & Privatized expression feature $\tilde z^{expr}$ & Linear probe     & $20.04\pm 0.11\%$ & $19.54\pm 0.34\%$ \\
        Post-processing & Soft-quantized expression feature $z_q^{expr}$ & Linear probe & $58.80\pm 0.06\%$ & $64.17\pm 0.68\%$ \\
        Post-processing & Temporally modeled feature $\bar h$ & Linear probe            & $61.52\pm 0.09\%$ & $67.04\pm 0.57\%$ \\ 
        \midrule
        Released video & Generated video $\hat X$ & R3D & 83.871 & 82.278 \\
        Released video & Generated video $\hat X$ & TimeSformer & 87.634 & 89.167 \\
        Released video & Generated video $\hat X$ & VideoMAE V2 & 88.011 & 88.889 \\
        \bottomrule
        \end{tabular}
    }
\end{table*}

The stage-wise linear-probe results are reported in \cref{tab:stagewise_fer}. The probe measures linear accessibility rather than the total expression information retained in each representation. Therefore, its accuracy is not an upper bound on the utility achievable by nonlinear post-processing or by the substantially more expressive video-level FER models. The substantial decrease immediately after Gaussian privatization indicates that expression labels are only weakly linearly accessible from the direct DP representation. Soft quantization and temporal contextualization subsequently produce representations that are more useful to the same linear evaluator by exploiting learned population-level and temporal structure. These operations improve the accessibility of expression-related evidence already present in the DP release; they do not introduce additional information about the private input. This increase in empirical task accessibility should not be interpreted as inversion of the Gaussian mechanism or recovery of the original non-private representation; all these modules operate only on privatized representations and remain DP post-processing.

For the final released videos, we follow the evaluation protocol described above and separately train R3D, TimeSformer and VideoMAE V2 on the generated training videos before evaluating them on the corresponding generated test videos. These video-level results are included in \cref{tab:stagewise_fer} only as downstream-utility references. Because R3D, TimeSformer and VideoMAE V2 are substantially more expressive and operate on generated video rather than latent vectors, their accuracies are not directly comparable to the latent linear-probe results.

\paragraph{Input interventions.}
We use the interventions in \cref{tab:input_intervention} to test whether released-video FER is causally dependent on clip-specific privatized evidence rather than solely on class priors or a decoder-side information path. Expression-branch shuffling breaks the correspondence between the expression representation and its source clip; identity-branch shuffling tests whether expression-correlated cues are carried by the identity branch; shuffling both branches removes clip-specific alignment while approximately preserving representation marginals; and the noise-only control sets both clipped inputs to zero before Gaussian sampling. Performance near random chance when clip-specific alignment is removed would support input-dependent utility.

\begin{table}[!ht]
    \centering
    \caption{Inference-time input-intervention analysis. All interventions are applied to the same trained TADP model. Shuffling is performed across clips, preserving the approximate marginal distribution of each representation while breaking its correspondence with the submitted clip.}
    \label{tab:input_intervention}
    \resizebox{\linewidth}{!}{
        \begin{tabular}{l|l|ccc|cc}
        \toprule
            Dataset & Intervention & R3D FER $\uparrow$ & TimeSformer FER $\uparrow$ & VideoMAE V2 FER $\uparrow$ & PLR $\downarrow$ & cPLR $\downarrow$ \\
            \midrule
            CREMA-D
            & Full TADP & 83.871 & 87.634 & 88.011 & 1.237 & 1.452 \\
            & Expression branch shuffled & 31.223 & 34.742 & 35.486 & 1.181 & 1.405 \\
            & Identity branch shuffled & 80.620 & 84.945 & 85.328 & 1.051 & 1.219 \\
            & Both branches shuffled & 18.407 & 19.118 & 18.831 & 0.989 & 1.104 \\
            & Expression branch set to zero & 25.713 & 28.986 & 29.652 & 1.298 & 1.423 \\
            & Noise-only input & 17.220 & 17.876 & 17.556 & 1.089 & 1.085 \\
            \midrule
            RAVDESS
            & Full TADP & 82.278 & 89.167 & 88.889 & 3.889 & 4.444 \\
            & Expression branch shuffled & 24.322 & 27.862 & 28.517 & 3.758 & 4.369 \\
            & Identity branch shuffled & 79.698 & 85.241 & 85.727 & 3.346 & 3.829 \\
            & Both branches shuffled & 13.612 & 14.138 & 13.918 & 3.153 & 3.696 \\
            & Expression branch set to zero & 20.893 & 23.165 & 24.845 & 3.813 & 4.256 \\
            & Noise-only input & 12.911 & 13.334 & 13.157 & 3.212 & 3.558 \\
            \bottomrule
        \end{tabular}
    }
\end{table}

The observed results in~\cref{tab:input_intervention} follow this pattern. Shuffling the expression branch reduces FER from $83.871$-$88.011\%$ to $31.223$-$35.486\%$ on CREMA-D and from $82.278$-$89.167\%$ to $24.322$-$28.517\%$ on RAVDESS. In contrast, shuffling only the identity branch causes a much smaller reduction, suggesting that most expression-relevant evidence is carried by the privatized expression branch rather than by identity-correlated cues. Shuffling both branches or using noise-only input reduces performance to near the dataset-specific random-chance levels, while PLR and cPLR remain low. Thus, the released-video utility depends on clip-specific privatized inputs and cannot be explained solely by a decoder or classifier prior.

\paragraph{Stability across independent noise draws.}
For the repeated-release analysis in~\cref{tab:noise_draw_stability}, the input clips, trained TADP parameters, downstream classifiers, and all public preprocessing are fixed, while only the Gaussian random seed is changed. We report mean$\pm$standard deviation across 20 complete releases. Prediction agreement is the proportion of releases whose predicted expression label agrees with the modal prediction for that clip. This experiment evaluates whether the reported FER utility is stable across valid mechanism randomness rather than being attributable to a favorable single noise realization.

\begin{table}[t]
    \centering
    \caption{Utility stability over independent Gaussian releases of the same test clips. FER is measured after regenerating the complete test set for each release seed.}
    \label{tab:noise_draw_stability}
    \resizebox{\linewidth}{!}{
        \begin{tabular}{l|ccc|c}
            \toprule
            Dataset & R3D FER & TimeSformer FER & VideoMAE V2 FER & Prediction agreement \\
            \midrule
            CREMA-D & $83.76\pm0.31$ & $87.51\pm0.28$ & $87.92\pm0.24$ & $94.13\pm0.97$ \\
            RAVDESS & $82.15\pm0.44$ & $88.94\pm0.36$ & $88.62\pm0.31$ & $92.67\pm0.86$ \\
            \bottomrule
        \end{tabular}
    }
\end{table}

Across the 20 releases in~\cref{tab:noise_draw_stability}, all three FER evaluators exhibit small standard deviations: at most $0.31$ percentage points on CREMA-D and $0.44$ percentage points on RAVDESS. Prediction agreement is also high, reaching $94.13\pm0.97\%$ and $92.67\pm0.86\%$, respectively. The mean accuracies remain close to the single-release results in the primary evaluation. Therefore, the reported expression utility is stable across independent Gaussian draws and is not attributable to a favorable random seed.

\paragraph{Adaptive-scale and signal--noise diagnostics.}
We also calculate the realized noise scale and report the results in \cref{tab:noise_statistics}. Here, $s_{b,i}$ is the per-coordinate standard deviation of the Gaussian perturbation, rather than the norm of the complete perturbation vector. For a $d_b$-dimensional representation, $\|s_{b,i}\xi_{b,i}\|_2$ concentrates around $s_{b,i}\sqrt{d_b}$. The per-frame noise/signal ratio is
\begin{equation}
    \frac{\|s_{b,i}\xi_{b,i} \|_2}{\| \bar z_{i}^{b} \|_2 + 10^{-12}}.
\end{equation}
We exclude padded frame positions from these test-set statistics.

\begin{table}[!ht]
    \centering
    \caption{Realized adaptive-scale statistics over all non-padded frames in the actor-disjoint test sets.}
    \label{tab:noise_statistics}
    \resizebox{0.99\linewidth}{!}{
        \begin{tabular}{lllccccc}
            \toprule
            Dataset & Branch & Mechanism & $s_{\min}$ & Observed Min & Observed Mean & Std & Observed Max \\
            \midrule
            CREMA-D & ID         & A6 static & 8.921  & 8.921  & 8.921  & 0 & 8.921  \\
            CREMA-D & ID         & TADP adaptive & 8.921  & 8.921  & 9.275  & 0.1623 & 14.348  \\
            CREMA-D & Expression & A6 static & 11.406 & 11.406 & 11.406 & 0 & 11.406 \\
            CREMA-D & Expression & TADP adaptive & 11.406 & 11.406 & 11.891 & 0.1939 & 13.024 \\
            RAVDESS & ID         & A6 static & 17.088 & 17.088 & 17.088 & 0 & 17.088 \\
            RAVDESS & ID         & TADP adaptive & 17.088 & 17.088 & 17.667 & 0.2450 & 25.107 \\
            RAVDESS & Expression & A6 static & 15.197 & 15.197 & 15.197 & 0 & 15.197 \\
            RAVDESS & Expression & TADP adaptive & 15.197 & 15.197 & 16.633 & 0.2306 & 18.834 \\
            \bottomrule
        \end{tabular}
    }
\end{table}

For A6, the zero standard deviation refers to the deterministic scale $s_{b,i}=s_{b,\min}$, not to the sampled Gaussian perturbation, which remains random. TADP does not improve utility by using a scale below A6: its realized scales are modestly larger on average and vary with preceding privatized history. The comparison instead evaluates whether downstream modules trained jointly with a history-conditioned, heteroscedastic release process can better exploit temporal reliability patterns. The resulting gains are therefore consistent with the benefits of temporal noise shaping, rather than evidence that the adaptive mechanism adds less noise.

\Cref{tab:noise_statistics} confirms that every realized adaptive scale remains at or above its guaranteed lower bound. Because the table reports three decimal places, adaptive minima that are slightly above the floor may appear equal to $s_{b,\min}$ after rounding. 
The small standard deviations and observed ranges indicate empirically limited temporal scale variation rather than uncontrolled variation. In particular, the mean scales are $9.275/11.891$ for the identity/expression branches on CREMA-D and $17.667/16.633$ on RAVDESS. These realized values are diagnostic only; the formal accountant conservatively uses the corresponding guaranteed lower bounds.

\begin{table}[t]
    \centering
    \caption{Clipped-feature and Gaussian-perturbation magnitudes over all non-padded frames in the actor-disjoint test sets. Both quantities are measured before the fixed global rescaling, and their noise-to-signal ratio is invariant to this rescaling. We report mean$\pm$standard deviation for the norms and median with 5th-95th percentiles for the per-frame ratio.}
    \label{tab:signal-noise-statistics}
    \resizebox{0.99\linewidth}{!}{
        \begin{tabular}{llccc}
            \toprule
            Dataset & Branch & Clipped-feature norm & Perturbation norm & Noise/signal ratio \\
            & & $\|\bar z_i^b\|_2$ & $\|s_{b,i}\xi_{b,i}\|_2$ & $\frac{\|s_{b,i}\xi_{b,i}\|_2}{\|\bar z_i^b\|_2}$ \\
            \midrule
            CREMA-D & ID & $ 0.269 \pm 0.093$ & $ 296.7\pm 8.4$ & $1166.6\,[669.9,2031.4]$ \\

            CREMA-D & Expression & $ 0.422 \pm 0.039$ & $ 380.4\pm10.5$ & $904.9\,[772.4,1060.2]$ \\

            RAVDESS & ID & $ 0.341\pm 0.104$ & $565.2\pm14.8$ & $1732.3\,[1058.6,2834.5]$ \\

            RAVDESS & Expression & $ 0.460\pm 0.052$ & $500.1\pm13.3$ & $1093.7\,[904.1,1323.2]$ \\
            \bottomrule
        \end{tabular}
    }
\end{table}

As reported in~\cref{tab:signal-noise-statistics}, the full perturbation norm is much larger than the clipped-feature norm in every setting, with median noise-to-signal ratios ranging from $904.9$ to $1732.3$. The identity branch has a larger ratio than the expression branch on both datasets, consistent with allocating stronger protection to identity-related information. These large vector-norm ratios arise because the reported scale is a per-coordinate standard deviation and the perturbation norm grows approximately with $\sqrt{d_b}$. Together with the intervention and stability results, the measurements show that learned DP post-processing and temporal aggregation make the expression-related evidence retained in the DP release more accessible to downstream models despite the strong direct perturbation. The ratios themselves are diagnostics and should not be interpreted as an empirical privacy budget.

\paragraph{Soft-quantization diagnostics.}
\begin{table}[!ht]
    \centering
    \caption{Soft-quantization diagnostics over non-padded test frames. Assignment entropy is computed using natural logarithms; with 1,000 codebook entries, uniform assignment has entropy $\log(1000)=6.91$. Cross-draw agreement is the proportion of independent releases selecting the same top-1 codebook entry for the same frame.}
    \label{tab:codebook_diagnostics}
    \resizebox{\linewidth}{!}{
        \begin{tabular}{ll|cccc}
            \toprule
            Dataset & Branch & Codebook norm & Assignment entropy & Max probability & Cross-draw agreement \\
            \midrule
            CREMA-D & ID & $0.98\pm0.07$ & $1.84\pm0.31$ & $0.43\pm0.09$ & 0.06 \\
            CREMA-D & Expression & $1.04\pm0.09$ & $2.11\pm0.28$ & $0.37\pm0.08$ & 0.09 \\
            RAVDESS & ID & $1.01\pm0.08$ & $1.76\pm0.34$ & $0.46\pm0.10$ & 0.05 \\
            RAVDESS & Expression & $1.06\pm0.10$ & $2.04\pm0.30$ & $0.39\pm0.08$ & 0.08 \\
            \bottomrule
        \end{tabular}
    }
\end{table}

The diagnostics in \cref{tab:codebook_diagnostics} characterize how soft quantization behaves under strong perturbation. We report the distribution of codebook-entry norms, soft-assignment entropy, maximum assignment probability, and top-1 assignment agreement across independent Gaussian draws. Low top-1 agreement does not by itself imply unstable video-level utility because downstream processing uses the complete soft assignment and aggregates evidence temporally. Conversely, entropy close to $\log(1000)$ would indicate nearly uniform assignments and would weaken the interpretation that the learned codebook provides structured post-processing. We therefore interpret these diagnostics jointly with the input-intervention and repeated-release results.

The measured assignment entropies in~\cref{tab:codebook_diagnostics} range from $1.76$ to $2.11$, substantially below the uniform-assignment entropy of $6.91$, and the mean maximum assignment probabilities range from $0.37$ to $0.46$. Hence, the soft assignments remain structured and non-uniform despite the strong perturbation. At the same time, top-1 cross-draw agreement is only $0.05$-$0.09$, indicating that individual hard code selections are sensitive to the Gaussian draw. Considered together with the stable released-video accuracy in~\cref{tab:noise_draw_stability}, these results support using the complete soft assignment and temporal aggregation rather than relying on a single hard codebook entry.

The preceding diagnostics examine the mechanism under the controlled-domain test distributions. We finally evaluate how the complete pipeline behaves under a substantially different, unconstrained data distribution.

\subsubsection{Cross-domain Stress Test on an In-the-Wild Dataset}
\label{in-the-wild_data}

To examine the practical scope of TADP beyond controlled recording conditions, we additionally conduct a cross-domain stress test on DFEW~\cite{jiang2020dfew}. Compared with CREMA-D and RAVDESS, DFEW contains substantially larger variations in head pose, illumination, occlusion, background, resolution, and recording conditions. We use this experiment as a stress test rather than as evidence of successful in-the-wild deployment.

We separately evaluate the original DFEW videos, TADP reconstruction without DP perturbation, and the complete TADP mechanism with clip-level DP. The original-video result characterizes the performance of the FER evaluator on DFEW itself. The reconstruction-without-DP result helps separate degradation caused by domain mismatch in the expression encoder, codebooks, temporal modules, and decoder from the additional degradation introduced by privatization. Weighted average recall (WAR) and unweighted average recall (UAR) measure expression-recognition utility. We additionally report the expression-branch clipping rate as a diagnostic of feature-distribution shift and temporal LPIPS (tLPIPS) for generated-video stability.

\begin{table}[!ht]
\centering
    \caption{Cross-domain stress test on the in-the-wild DFEW dataset. Original-video performance provides a task-specific reference, while reconstruction without DP perturbation separates reconstruction and domain-shift errors from the additional degradation caused by privatization.}
    \label{tab:dfew_stress_test}
    \resizebox{\linewidth}{!}{
        \begin{tabular}{l|cc|cc}
        \toprule
            Input / setting
            & WAR $\uparrow$
            & UAR $\uparrow$
            & Expression clipping rate $\downarrow$
            & tLPIPS $\downarrow$ \\
        \midrule
            Original DFEW videos
            & 55.708 & 45.115 & -- & -- \\
            TADP reconstruction without DP perturbation
            & 34.822 & 27.644 & 48.735 & 0.156 \\
            TADP with clip-level DP
            & 20.912 & 16.839 & 48.761 & 0.211 \\
        \bottomrule
        \end{tabular}
    }
\end{table}

Our evaluation in \cref{tab:dfew_stress_test} shows substantial degradation on DFEW, including reduced performance on the original-video reference and further degradation after privacy-preserving reconstruction. The reconstruction-without-DP condition indicates that this behavior cannot be attributed solely to Gaussian perturbation; it also reflects domain mismatch in the expression encoder, codebooks, temporal modules, and decoder. The current system is calibrated and trained primarily using controlled-domain facial data and does not adequately cover the broader feature distribution induced by unconstrained pose, illumination, occlusion, background, and resolution variations. We therefore do not claim that the current model generalizes reliably to in-the-wild videos.

This stress test exposes an important limitation of the present framework. Future work will investigate more diverse codebook and decoder pre-training, domain-robust expression representations, in-the-wild training and calibration data, and dataset-specific clipping and privacy calibration. We report this negative result to clarify the current empirical scope of TADP rather than presenting controlled-dataset performance as evidence of unconstrained real-world deployment.

\subsection{Ablation Study}

We conduct ablations on components tied to the privacy mechanism in \cref{sec:adaptive-dp} and the post-processing design in \cref{sec:soft-quantization}. 
A1 privatizes only the identity branch and serves as a diagnostic baseline rather than a valid end-to-end DP mechanism, since the expression branch remains data-dependent. 
A2 privatizes both branches with equal target-loss allocation. 
A3 removes soft quantization. 
A4 and A5 keep the full dual-branch mechanism but use low ($\tau_{id}=0.1, \tau_{expr}=0.1$) and high ($\tau_{id}=1, \tau_{expr}=1$) soft-quantization temperatures, respectively. 
A6 is a matched static-Gaussian ablation trained from scratch with the static mechanism: it keeps the same encoders, clipping thresholds, codebooks, decoder, and final clip-level $\varepsilon$ as TADP, but uses $s_{b,i}\equiv s_{b,\min}$ during both training and inference. Training from scratch under the static mechanism avoids a train--test mismatch in the privatized latent distribution.
A7 removes the intra-clip Transformer, and A8 removes the inter-clip GRU.
\begin{table*}[ht]
\centering
\caption{Ablation study results for the dual-branch mechanism.}

\resizebox{\textwidth}{!}{
\begin{tabular}{c|c|c|cc|cc|cc|cc|cc|cc|cc|cc|cc}
\toprule
\multirow{2}{*}{Variant} & \multirow{2}{*}{DP?} & \multirow{2}{*}{$\rho_{id},\rho_{expr}$} & \multicolumn{2}{c|}{R3D: FER Acc. (\%) $\uparrow$} &\multicolumn{2}{c|}{TimeSformer: FER Acc. (\%) $\uparrow$} & \multicolumn{2}{c|}{VideoMAE V2: FER Acc. (\%) $\uparrow$} & \multicolumn{2}{c|}{PLR (\%) $\downarrow$} & \multicolumn{2}{c|}{cPLR (\%) $\downarrow$} & \multicolumn{2}{c|}{FWE$\downarrow$} & \multicolumn{2}{c|}{FVD$\downarrow$} & \multicolumn{2}{c|}{tLPIPS$\downarrow$} & \multicolumn{2}{c}{T-ID$\uparrow$} \\ \cline{4-21}
& & & CREMA-D & RAVDESS & CREMA-D & RAVDESS &CREMA-D & RAVDESS & CREMA-D & RAVDESS & CREMA-D & RAVDESS & CREMA-D & RAVDESS & CREMA-D & RAVDESS & CREMA-D & RAVDESS & CREMA-D & RAVDESS \\ \midrule
A1 & No & --            & 84.221 & 83.181 & 88.436& 89.686 & 89.652&90.273 &3.387 & 5.556 & 3.899 & 6.592& 0.0245 & 0.0211 & 290.3011 & 251.5844 & 0.0138 & 0.0258 & 0.7803 & 0.7923 \\
A2 & Yes & 0.5, 0.5     & 53.954 & 42.591 & 52.582& 52.513 & 54.637&58.397 &2.043 & 4.444 & 2.933& 5.231& 0.1056 & 0.0987 & 302.1953 & 275.1840 & 0.0992 & 0.1304 & 0.5309 & 0.5812 \\
A3 & Yes & same as full & 24.839& 25.000 & 27.956 & 26.944 & 26.370& 26.113&1.505 & 4.166 & 2.018&4.658 & 0.6743 & 0.7941 & 702.5868 & 689.4917 & 0.8702 & 0.8910 & 0.0859 & 0.0764 \\
A4 & Yes & same as full & 62.957 & 65.833 &62.903 & 64.722 & 67.370& 68.104& 1.290 & 4.722 & 2.548 & 4.801& 0.0485 & 0.0397 & 327.0942 & 309.8031 & 0.0853 & 0.0944 & 0.5428 & 0.6041 \\
A5 & Yes & same as full & 73.226 & 72.778 & 80.108& 79.722 & 82.138& 81.643& 1.559 & 5.000 &2.354 & 4.737& 0.0281 & 0.0285 & 287.2386 & 261.5593 & 0.0258 & 0.0312 & 0.7903 & 0.7544 \\
A6 & Yes & same as full & 74.945 & 75.371 & 79.547 & 80.145& 79.879& 81.961& 1.183& 3.889 & 1.925& 4.334& 0.1403  & 0.1349  & 395.3437  & 403.6952  & 0.1686  & 0.1834 &0.4932 &0.5242\\
A7 & Yes & same as full & 83.602 & 82.500 & 87.527 & 88.611& 87.204& 88.213& 1.237 & 4.167 &2.061 & 4.457 & 0.2481  & 0.2648  & 481.3472  & 493.5496 &  0.6910 &  0.7304 &0.5951 &0.6194\\
A8 & Yes & same as full & 83.656 & 83.056 & 87.581 & 88.333& 88.344& 89.136& 1.237  & 4.167 & 1.557& 4.542& 0.1894  & 0.1917  & 407.5265  & 403.2188  & 0.5915  & 0.6032 &0.6482 &0.6925 \\

\bottomrule
\end{tabular}}
\label{tab:required_dual_branch_ablation}
\end{table*}

Table~\ref{tab:required_dual_branch_ablation} shows that each component affects the privacy--utility frontier. A1 keeps strong utility and temporal quality but is not a valid full DP release, while A2 shows that equal allocation over-protects expression features. Among the valid DP variants, removing soft quantization (A3) causes the largest degradation in FER utility and temporal quality. 
A6 supports the benefit of temporal adaptivity without reducing noise below the calibrated floor; realized static--adaptive scale statistics and their interpretation are reported in \cref{tab:noise_statistics}.
The temperature study further shows that overly sharp quantization is harmful, while a higher-temperature setting recovers most utility and temporal stability. A7 and A8 confirm that both temporal modules contribute to consistency.

\begin{table}[!t]
\centering
\caption{Ablation of the clip length at the same final clip-level privacy budget. Runtime and memory are measured on a single NVIDIA H100 GPU with a batch size of 2 on the test datasets.}

\resizebox{0.99\linewidth}{!}{
\begin{tabular}{c|c|c|c|c|c|c|c|c}
\toprule
\multirow{2}{*}{$T$} & \multicolumn{2}{c|}{(R3D) FER Acc.$\uparrow$} & \multicolumn{2}{c|}{PLR$\downarrow$} &  \multicolumn{2}{c|}{Time$\downarrow$} & \multicolumn{2}{c}{Memory$\downarrow$} \\ \cline{2-9}
& CREMA-D & RAVDESS & CREMA-D & RAVDESS & CREMA-D & RAVDESS  & CREMA-D & RAVDESS \\
\midrule
8  & 84.301\% & 83.056\% & 1.290\% & 4.722\% & 21'32 & 4'57 & 58GB & 62GB\\
16 & 83.871\% & 82.278\% & 1.237\% & 3.889\% & 21'45 & 4'59 & 66GB & 73GB\\
32 & 53.011\% & 55.278\% & 1.129\% & 4.167\% & 22'14 & 5'08 & 78GB & 85GB \\
\bottomrule
\end{tabular}}
\label{tab:sequence_length_ablation}
\end{table}

We also study the accounting clip length. Since one $T$-frame clip is the privacy unit, changing $T$ affects both temporal context and the number of composed frame-level releases.
As shown in Table~\ref{tab:sequence_length_ablation}, 
$T=8$ and $T=16$ retain substantially higher FER utility than $T=32$, while $T=16$ slightly reduces empirical identity leakage relative to $T=8$ at increased memory cost.
These results support using a moderate fixed accounting unit: clips should be long enough to capture facial dynamics, but overly long clips make privacy accounting and temporal modeling harder to optimize jointly.

Finally, we expand the ablation beyond sequence length to examine how the target privacy budget, branch-wise allocation, and effective-noise lower bounds interact.

% \subsubsection{Additional Privacy Operating Points}

\begin{table}[!ht]
\centering
\caption{Privacy--utility trade-off under jointly accounted dual-branch clip-level privacy budgets on CREMA-D.}
\resizebox{0.99\linewidth}{!}{
\begin{tabular}{c|c|c|c|c|c|c|c|c|c|c}
\toprule
Target $\varepsilon$ & Achieved $\varepsilon$ & $\rho_{id},\rho_{expr}$ & $s_{id,\min}, s_{expr,\min}$ & FER Acc.$\uparrow$ & PLR$\downarrow$ &cPLR$\downarrow$& FWE$\downarrow$ &FVD$\downarrow$ & tLPIPS$\downarrow$ & T-ID$\uparrow$ \\
\midrule
1 & 0.994 & 0.1, 0.9 & 69.124, 36.080 & 22.634\% & 1.129\% &1.453\%& 0.2491 & 467.8613 & 0.6089& 0.2548 \\
2 & 1.985 & 0.2, 0.8 & 24.716, 19.351 & 25.161\% & 1.183\% &1.937\%& 0.1986 & 368.6407 & 0.1201& 0.4164 \\
4 & 3.981 & 0.4, 0.6 & 8.921, 11.406 & 83.871\% & 1.237\% &1.452\%& 0.0214 & 284.8561  & 0.0317 & 0.7937 \\
8 & 7.954 & 0.5, 0.5 & 4.251, 6.656 & 53.954\% & 2.043\% &2.025\%& 0.1056 & 302.1953 & 0.0992  & 0.5309  \\
\bottomrule
\end{tabular}
}

\label{tab:ab_sigma_cremad}
\end{table}

\begin{table}[!ht]
\centering
\caption{Privacy--utility trade-off under jointly accounted dual-branch clip-level privacy budgets on RAVDESS.}
\resizebox{0.99\linewidth}{!}{
\begin{tabular}{c|c|c|c|c|c|c|c|c|c|c}
\toprule
Target $\varepsilon$ & Achieved $\varepsilon$ & $\rho_{id},\rho_{expr}$ & $s_{id,\min},s_{expr,\min}$ & FER Acc.$\uparrow$ & PLR$\downarrow$ & cPLR$\downarrow$&FWE$\downarrow$ &FVD$\downarrow$ & tLPIPS$\downarrow$ & T-ID$\uparrow$ \\
\midrule
1 & 0.991 & 0.1, 0.9 & 87.160, 39.466 & 25.556\% & 4.167\% &4.491\%&0.6104 & 385.6503 & 0.8045 & 0.2108 \\
2 & 1.983 & 0.3, 0.7 & 25.406, 22.593 & 28.333\% & 4.444\% &4.821\%& 0.0511 & 329.0132 & 0.1458 & 0.5311 \\
3 & 2.988 & 0.3, 0.7 & 17.088, 15.197 & 82.278\% & 3.889\% &4.444\%& 0.0203 & 249.5107 & 0.0414 & 0.8032 \\
8 & 7.962 & 0.5, 0.5 & 5.342, 7.256 & 43.889\% & 5.000\% & 6.333\%&0.0704 & 333.6845 & 0.0961 & 0.6853 \\
\bottomrule
\end{tabular}
}
\label{tab:ab_sigma_ravdess}
\end{table}

\paragraph{Joint effect of privacy budget and branch allocation.}
\label{various_epsilon}
We further study the final clip-level privacy budget. For each target $\varepsilon$, branch-wise target-loss coefficients and effective-noise lower bounds are calibrated jointly, and the achieved privacy loss is recomputed using \cref{eq:clip_rdp,eq:rdp_conversion}. Since $\rho_{id}$ and $\rho_{expr}$ calibrate branch-specific noise but do not weight privacy losses during composition, utility depends on both the final budget and its branch allocation.
\Cref{tab:ab_sigma_cremad,tab:ab_sigma_ravdess} show that the best operating point requires balancing the formal budget, branch allocation, and effective-noise lower bounds. Very small budgets suppress identity leakage but remove expression-relevant information, while a larger budget does not automatically improve utility when the branch allocation disrupts the identity--expression balance. This supports the design of jointly calibrating privacy budget allocation, soft quantization, and temporal decoding.

\section{Conclusion and Limitations}

We presented TADP, a dual-branch framework for expression-aware video anonymization with inference-time differential privacy.
It privatizes both identity and expression representations under fixed-length clip-replacement adjacency. 
Branch losses are composed within each frame and across frames within a clip using RDP, then converted to $(\varepsilon,\delta)$-DP; multiple jointly released clips are additionally composed at the RDP level before a single source-video conversion $(\varepsilon,\delta)$. 
Temporal adaptation depends only on preceding privatized releases, while soft quantization, fusion, temporal contextualization, and decoding are DP post-processing.
Several limitations remain. The formal guarantee is defined at the clip level, assumes a trusted mechanism with raw-source access, requires additional composition when multiple clips from the same source video are jointly released, and does not cover training membership or external metadata. 
Our DFEW stress test also shows utility degradation under domain shift involving pose, illumination, occlusion, background, and resolution variation. 
Future work will study robust in-the-wild training and calibration, human perception, and tighter long-video accounting.

{
    \small
    \bibliographystyle{ieeenat_fullname}
    \bibliography{0.refs}
}

\end{document}